\documentclass[10pt,twocolumn,letterpaper]{article}

\usepackage[pagenumbers]{cvpr} 

\usepackage{url}
\usepackage{amssymb}
\usepackage[T1]{fontenc}
\usepackage{amsfonts}
\usepackage[ruled,linesnumbered]{algorithm2e}
\usepackage{algpseudocode}
\usepackage{mathrsfs}
\usepackage{bbm}
\usepackage{amsthm}
\usepackage{color,xcolor}
\usepackage{etoolbox}
\usepackage{amsmath}
\usepackage{graphicx}
\usepackage{subcaption}
\usepackage{booktabs}
\usepackage{epstopdf}
\usepackage{booktabs}

\usepackage{amsmath,amsfonts,bm}

\def\eqref#1{equation~\ref{#1}}

\def\1{\bm{1}}

\def\vmu{{\bm{\mu}}}

\def\mSigma{{\bm{\Sigma}}}

\DeclareMathAlphabet{\mathsfit}{\encodingdefault}{\sfdefault}{m}{sl}
\SetMathAlphabet{\mathsfit}{bold}{\encodingdefault}{\sfdefault}{bx}{n}

\newcommand{\R}{\mathbb{R}}

\usepackage[pagebackref,breaklinks,colorlinks]{hyperref}

\usepackage[capitalize]{cleveref}

\newtheorem{theorem}{Theorem}
\newtheorem{assumption}{Assumption}

\newtheorem{remark}{Remark}
\newtheorem{example}{Example}
\newtheorem{definition}[theorem]{Definition}
\crefname{section}{Sec.}{Secs.}
\Crefname{section}{Section}{Sections}
\Crefname{table}{Table}{Tables}
\crefname{table}{Tab.}{Tabs.}

\def\cvprPaperID{10567} 
\def\confName{CVPR}
\def\confYear{2022}

\usepackage{overpic}
\usepackage{enumitem} 
\usepackage{overpic} 
\usepackage{color}

\definecolor{turquoise}{cmyk}{0.65,0,0.1,0.3}
\definecolor{purple}{rgb}{0.65,0,0.65}
\definecolor{dark_green}{rgb}{0, 0.5, 0}
\definecolor{orange}{rgb}{0.8, 0.6, 0.2}
\definecolor{red}{rgb}{0.8, 0.2, 0.2}
\definecolor{darkred}{rgb}{0.6, 0.1, 0.05}
\definecolor{blueish}{rgb}{0.0, 0.3, .6}
\definecolor{light_gray}{rgb}{0.7, 0.7, .7}
\definecolor{pink}{rgb}{1, 0, 1}
\definecolor{greyblue}{rgb}{0.25, 0.25, 1}

\DeclareMathOperator*{\argmax}{arg\,max}

\usepackage{blindtext}

\renewcommand{\paragraph}[1]{\vspace{1em}\noindent\textbf{#1}.}
\begin{document}

\title{Towards Understanding the Impact of Model Size on Differential Private Classification
}

\author{Yinchen Shen\\
Department of Mathematics\\
Sichuan University\\
Chengdu, Sichuan 610064, China \\
\texttt{1207102896@qq.com} \\
\and
Zhiguo Wang \\
Department of Mathematics\\
Sichuan University\\
Chengdu, Sichuan 610064, China \\
\texttt{wangzhiguo@scu.edu.cn} \\
\and
Ruoyu Sun \\
Department
of Industrial and Enterprise Systems
Engineering\\
University of Illinois at Urbana-Champaign \\
Urbana, IL  61801-2925, USA\\
\texttt{ruoyus@illinois.edu} \\
\and
Xiaojing Shen \\
Department of Mathematics\\
Sichuan University\\
Chengdu, Sichuan 610064, China \\
\texttt{shenxj@scu.edu.cn} \\
}

\maketitle
\begin{abstract}
Differential privacy (DP) is an essential technique for privacy-preserving. It was found that a large model trained for privacy preserving performs worse than a smaller model (e.g. ResNet50 performs worse than ResNet18). To better understand this phenomenon, we study high dimensional DP learning from the viewpoint of generalization. Theoretically, we show that for the simple Gaussian model \cite{schmidt2018adversarially}
with even small DP noise, if the dimension is large enough,
then the classification error can be as bad as the random guessing.
Then we propose a feature selection method to reduce the size of the model, based on a new metric which trades off the classification accuracy and privacy preserving. Experiments on real data support our theoretical results and demonstrate the advantage of the proposed method. 
\end{abstract}
\vspace{-0.3cm}
\section{Introduction}\label{sec_1}
Deep neural networks have made a series of remarkable achievements in the field of image recognition and classification, natural language processing. But training deep neural networks typically requires large and representative data to achieve high-performance \cite{gheisari2017survey}. Since the datasets often contain some sensitive information, such as medical records, location and purchase history, when we use these sensitive data to train a model without specific measures to the
secret information, individual privacy can be leaked \cite{fung2010privacy}. Thus, privacy-preserving is a crucial issue in deep learning.

\begin{figure*}[t]
\vspace{-0.5cm}
 \centering
\begin{subfigure}[b]{0.45\textwidth}
\includegraphics[width=\textwidth]{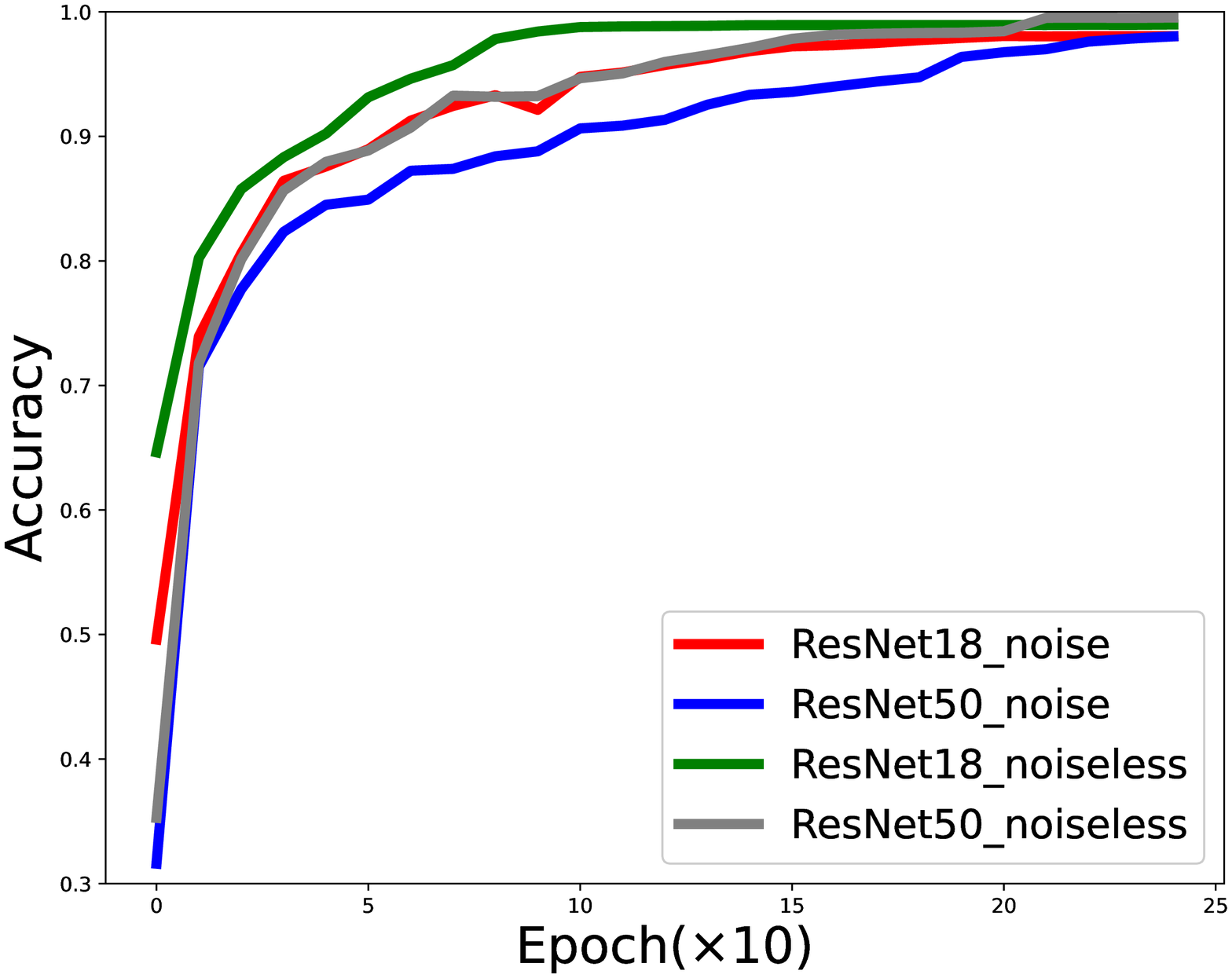}
\caption{Accuracy in training set}\label{fig_1}
\end{subfigure}
\begin{subfigure}[b]{0.45\textwidth}
\includegraphics[width=\textwidth]{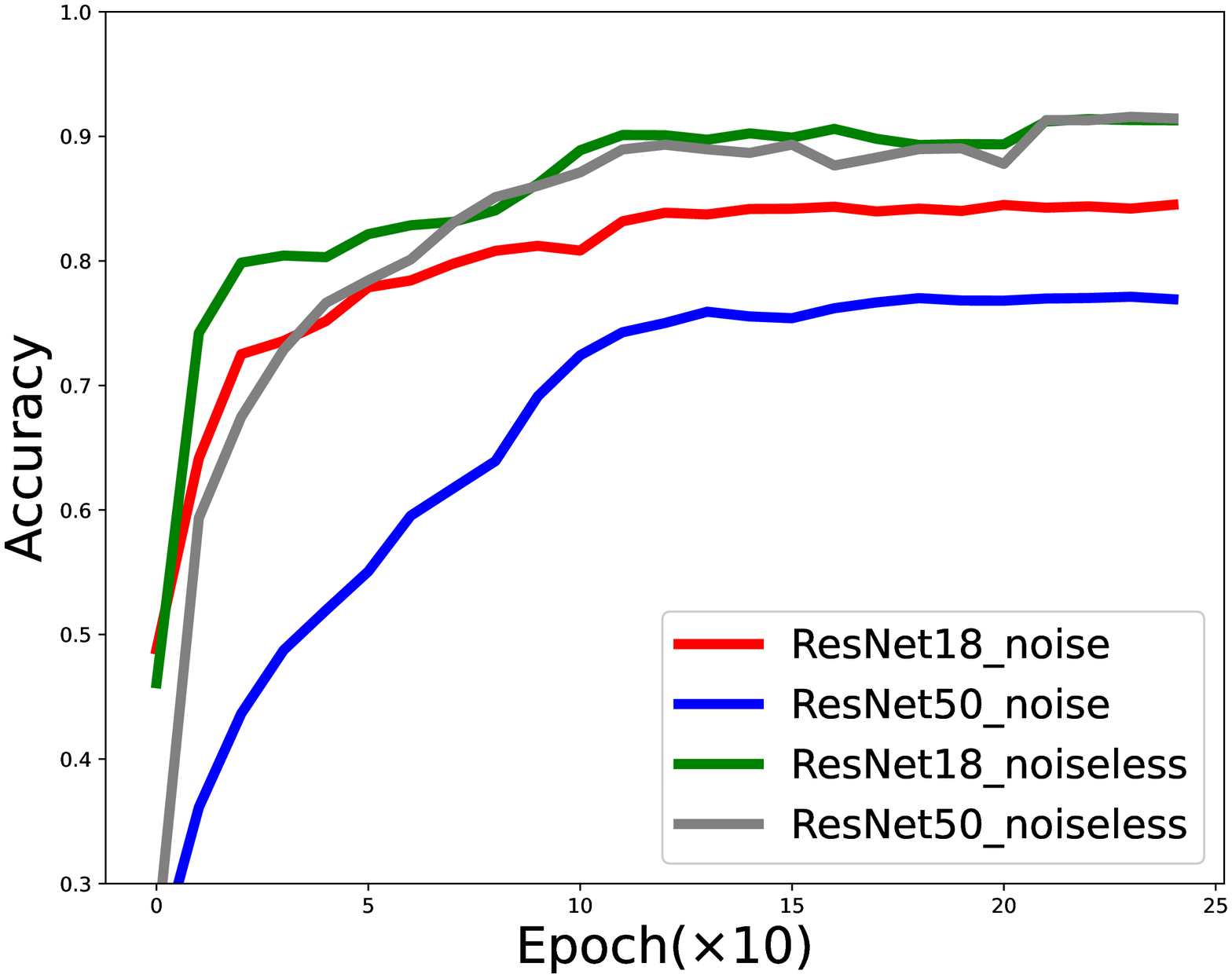}
\caption{Accuracy in test set}\label{fig_2}
\end{subfigure}
\caption{The performance of ResNet on CIFAR-10 by DP-SGD with $\varepsilon=5, \delta=0.0001$. (a) is the result in the training set, we see that both the ResNet 18 and ResNet 50 with noise or without noise obtain  98\% classification accuracy, respectively. (b) is the result in the test set, we see that the performance of ResNet 50 and ResNet 18 under noiseless condition is the same but ResNet50 causes much lower test accuracy than ResNet18.
}\label{fig_11}
\end{figure*}

One of the most popular techniques for privacy-preserving is $(\varepsilon, \delta)$-DP (differential privacy) that was first proposed by \cite{dwork2014algorithmic}.
A common mechanism to achieve DP is adding randomness (e.g. adding noise) to the data.
 Due to the simplicity, the method of adding
randomness has been extended to many settings
including deep learning
\cite{6817512,goodfellow2016deep,dupuy2021efficient}.


To understand how DP-SGD performs when the model
size changes, we trained DP-SGD on ResNet50 and ResNet18, respectively.
Fig.\ref{fig_11} shows the training
and test accuracy of ResNet50 and ResNet18,
trained by standard SGD (let us call them ResNet50-noiseless and ResNet18-noiseless) and
DP-SGD (call them ResNet50-noise and ResNet18-noise) respectively.
The left figure shows that
the training accuracy differs only a little
for all four models.
The right figure shows
that the test accuracy of ResNet50-noise is significantly lower
than that of ResNet18-noise.
This figure indicates two things:
first, the added noise
in DP-SGD is more detrimental to large
models than small models;
second, the bad performance of ResNet50-noise
is mainly due to generalization issue,
not optimization issue.


The above observation motivates us to ask the following question:
\begin{align*}
    &\text{\emph{Why do larger models under DP cause}}\\ &\text{\emph{ lower classification accuracy?}}
\end{align*}
We will answer this question from the generalization aspects of differential private learning.
In addition, we propose to select a subset of features to trade off the classification accuracy and privacy-preserving.

\subsection{Our Contributions}
\begin{itemize}
    \item \textbf{Generalization bound.} We analyze generalization error bound in a simple Gaussian model under DP.
By focusing on specific Gaussian noise, we can establish information-theoretic upper bounds of the classification error, which depends on the size of dimension and noise. The intuition is simple: as the dimension increases, noise can accumulate to cause classification error increase;  when the dimension is large enough, the classifier performs nearly the same as random guessing.
This provides an explanation why a larger model
causes lower classification accuracy under DP.
Our contribution is to provide a concrete analysis to formalize the intuition.

\item \textbf{Feature selection.} Since models have increasing classification error with an increasing number of dimensions,  we use the feature selection technique to reduce the dimension. A novel filter feature selection method is proposed, which uses a distance measure to assign a scoring to each feature. Comparing with t-statistic, the proposed method can obtain the stable and important features under DP.
\item \textbf{Experiment.} We perform simulation based on synthetic data and common real data such as RCV1, CIFAR-10.
 After using the proposed feature selection method, we show that  ResNet50 performs better than ResNet18 on CIFAR-10 in terms of DP.
\end{itemize}





\subsection{Related Works}

\textbf{Differential Privacy:} In \cite{xu2019laplace}, it considers both input-DP which adds noise on data processing, and output-DP which perturbs the answer of questions, and propose practical algorithms to show how to deal with two DP mechanisms.
 For the complicated situation like neural network, DP-SGD has been proved in utility \cite{chen2020understanding} with bounds for convergence after clipping gradient. Considering dimensions, \cite{2014Differentially} points that under assumptions of loss function and parameters, empirical risk can degenerate with dimension increment under differential privacy. Recently on neural network, \cite{tramer2020differentially} shows that linear models trained on
handcrafted features significantly outperform  neural networks for
moderate privacy budgets.
 However, they did not consider and set experiments for the affect of the dimension for the same type of model with accuracy instead of empirical risk.

\textbf{High Dimension Low Sample Size Data:}
 In low sample size $n$ and high dimension $p$, \cite{hall2005geometric} studies the impact of the increasing $n$  with fixed $p$, and they propose a geometric representation method for high-dimension data. For a linear model, \cite{2003Optimal} propose a similar assumption with our condition and achieve a  risk bound.  For a neural network, \cite{liu2017deep} propose DNP network to train on low sample by dropouts. DNP trains model by dropping neutrons randomly to minimize model size to increase model stability. Their works are powerful but in clean data, not concerning about privacy which people concerns.

\textbf{Feature Selection:}
 There are many traditional methods like wrapper and filter \cite{hart2000pattern} to select `important' features for the clean data.  Considering utility, the robustness of selection has been considered in \cite{ilyas2019adversarial}. They propose an algorithm to separate features with robustness in a certain model by adversary perturbation: changing labels for classes. However, their work either bases on clean data or adversary perturbation, which is
not suitable for DP.

\subsection{Outline of The Paper}
In the next section, we give some definitions and preliminaries. In Section 3, we analyze a simple Gaussian model and prove that larger models lead to higher error under DP. Then we proposed a feature selection algorithm for dimension reduction in differential privacy. The simulation in Section 4 reveals that feature selection can improve the performance and the proposed method performs better in some real dataset including RCV1 and CIRAR-10.


\section{Basic Definitions}
In this section, we first define $(\varepsilon, \delta)$-DP. Moreover, we consider a simple Gaussian model under DP. Then, We will analyze a  linear discriminant analysis (LDA) classifier for this Gaussian model.

\begin{definition}\label{DP1}
(Differential Privacy \cite{dwork2008differential}) A randomized algorithm $\mathcal{M}$ with domain dataset $\mathcal{D}$ is $(\varepsilon, \delta)$-differential private if for all $\mathcal{S} \subseteq$ Range $(\mathcal{M})$ and for all $x, y \in \mathcal{D}$ that $\|x-y\|_{1} \leq 1$ :
\begin{align}
\operatorname{Pr}[\mathcal{M}(x) \in \mathcal{S}] \leq \exp (\varepsilon) \operatorname{Pr}[\mathcal{M}(y) \in \mathcal{S}]+\delta.
\end{align}
\end{definition}
Since Definition \ref{DP1} imposes no limitations on randomized algorithm $\mathcal{M}$,
we use the following Gaussian mechanism that adding Gaussian noise, which we can create a DP algorithm for function $f$ with sensitivity
$\Delta f \triangleq \max\| f(d_i) - f(d_j) \|_1$, where the maximum is over all pairs of datasets $d_i$ and $d_j$ in dataset $\mathcal{D}$ differing in at most one element and $\|\cdot\|_1$ denotes the $\ell_1$ norm.
\begin{definition} \label{DP}
(Gaussian Mechanism \cite{dwork2014algorithmic})
Given any function $f:$ $\mathcal{D} \rightarrow \mathbb{R}^{k}$, the $(\epsilon, \delta) $-Gaussian mechanism is defined as:
\begin{align}
\mathcal{M}_{L}(x, f(\cdot), \varepsilon)=f(x)+\left(Y_{1}, \ldots, Y_{k}\right)
\end{align}
where $Y_{i}$ are i.i.d. random variables drawn from $\mathcal{N}(0,\sigma^2)$ where $\sigma = \Delta f\cdot \ln(1/ \delta )/ \epsilon $.
\end{definition}


Consider the $p$-dimensional classification problem between two classes $C_1$ and $C_2$. Suppose our clean data comes from the Gaussian  model (GM) \cite{chen2020more}, \cite{schmidt2018adversarially} . To analyze the impact of DP, based on the Gaussian mechanism, we consider the following binary classification by adding Gaussian noise to achieve $(\epsilon, \delta)$-DP.

\begin{definition}\label{PGMM}(Private GM)
Let $\mu_k \in \mathbb{R}^{p}$, $k=1,2$, be the per-class mean vector and
\begin{align}\label{eqn_sigma}
    \mathbf{\Sigma} \triangleq {\rm diag} (\sigma_1^2, ..., \sigma_p^2)
\end{align}
be the variance parameter. $(\epsilon, \delta )-$private Gaussian mixture model is defined by the following distribution over $(\hat{x}_k,k)\in \mathbb{R}^{p}\times\{1,2\}$:  First, draw a label $k$ from $\{1,2\}$ uniformly at random, then sample the data point $x_k\in \R^p$ from $\mathcal{N}(\mu_k, \mathbf{\Sigma})$. Then we get a non-private dataset $\{x_{k}^i,k\}$, $k=1,2$, $i=1,\ldots,n_k$. Finally, according Gaussian mechanism to obtain dataset  $\{\hat{x}_k^i,k\}$, where
\begin{align}\label{hat_x}
   \hat{x}_k^i = x_k^i +
   2C_p\ln(1/ \delta)/ \epsilon \cdot (\eta_1, ..., \eta_p),
\end{align}
where $\eta_i$ are i.i.d variables $\eta_i \sim \mathcal{N}(0,1)$ and $C_p \triangleq \max_{k \in \{1,2\}, i \leq n_k} \| x_k^i \|_1$ is a constant depending on dimension $p$.


\end{definition}\label{private_GMM}
From private GM, we can obtain some training data
$\{\hat{x}_{k}^i,k\}$, $k=1,2$, $i=1,\ldots,n_k$. Let $n=n_1+n_2$. Using these training data, the parameters $\mu_k$ and $\Sigma$ can be  estimated by
 \begin{align}\label{mum}
    &\hat{\mu}_{k}=  \frac{1} { n_{k}}\sum_{i=1}^{n_{k}} \hat{x}_{k}^i , k=1,2,  \\ \label{mum1}&\hat{\mathbf{\Sigma}}=\operatorname{diag}\left\{\frac{\left(S_{1 j}^{2}+S_{2 j}^{2}\right)}{2}, j=1, \ldots, p\right\},
\end{align}
where $ S_{k j}^{2}=\frac{1}{\left(n_{k}-1\right)}\sum_{i=1}^{n_{k}}\left(\hat{x}_{k j}^i-\bar{x}_{k j}\right)^{2}  $ is the sample variance of the $j$-th feature in class $k$ and $\bar{x}_{k j}=\frac{1} { n_{k}}\sum_{i=1}^{n_{k}} \hat{x}_{kj}^i$.

When $\mu_k$ and $\Sigma$ are known, the Fisher linear discriminant rule
\begin{align}\label{FLs}
    {\delta}_n(x) = (x-{\mu})\hat{\mathbf{\Sigma}}^{-1} {\alpha},
\end{align}
is the optimal classifier \cite{hao2015sparsifying}, where $\mu=\frac{1}{2}(\mu_1+\mu_2),\alpha=\mu_1-\mu_2$. In practice, these parameters are unknown and replaced by their estimates (\ref{mum})-(\ref{mum1}). Thus, the standard LDA using an empirical version of (\ref{FLs}) is defined as follows.
\begin{definition}(LDA classifier \cite{hart2000pattern}) \label{Fisher_classifier}
The LDA classifier is defined as:
\begin{align}
    \hat{\delta}_n(x) = (x-\hat{\mu})\hat{\mathbf{\Sigma}}^{-1} \hat{\alpha},
\end{align}
where $\hat{\mu}=\frac{1}{2}\left(\hat{\mu}_{1}+\hat{\mu}_{2}\right), \hat{\alpha}= \hat{\mu}_{1} - \hat{\mu}_{2}$.
\end{definition}
From Definition \ref{Fisher_classifier}, it shows that if $\hat{\delta}_n (x)>0$, which classifies sample $x$ into class $C_1$. Let us denote the parameter by $\theta=(\mu_1,\mu_2,\Sigma)$, we define the following classification error.
\begin{definition} (Classification Error) If we have a new observation $x$ from class $C_1$, then
the classification error $\mathbf{W}(\hat{\delta}_n,\theta)$ of the LDA classifier is defined by
\begin{align}
    \mathbf{W}(\hat{\delta}_n,\theta) \triangleq P(\hat{\delta}_n(x) \leq 0|\hat{x}_{k}^i)=1-\Phi(\Psi),
\end{align}
where $k=1,2,i=1,\ldots,n_k$,
\begin{align}
    \Psi=\frac{\left(\mu_{1}-\hat{\mu}\right) \prime \hat{\mathbf{\Sigma}}^{-1}\left(\hat{\mu}_{1}-\hat{\mu}_{2}\right)}{\sqrt{\left(\hat{\mu}_{1}-\hat{\mu}_{2}\right) \prime \hat{\mathbf{\Sigma}}^{-1}\left(\hat{\mu}_{1}-\hat{\mu}_{2}\right)}}.
\end{align}

\end{definition}

\section{Theoretical Results}\label{sec_2}





In this section, we first prove that with added noise, the error increases as the dimension increases. The intuition is that noise for different features can accumulate to cause large classification error. Then we focus on a criterion suitable for feature selection under DP to reduce dimension. Finally, we give an algorithm to realize our criterion for a dataset.

\subsection{Impact of High Dimension Under DP}
In this part, we first give an upper bound for the binary classification error.  Without loss of generality, the sample data are assumed to be balanced.




\begin{theorem} \label{thm_1} Suppose the training data comes from private GM (Definition \ref{private_GMM}) and  $n_1 = n_2$. In addition, assume
 $\log p=o(n), n=o(p)$.
Then the classification error $\mathbf{W}(\hat{\delta}_n, \theta)$ is bounded by
\begin{align}\label{WW}
\mathbf{W}(\hat{\delta}_n, \theta) \leq 1-\Phi\left( \frac{ \left(1+o_{p}(1)\right)\Gamma}{2\left[\frac{4p}{n}+ \left(1+o_{p}(1)  \right)\Gamma\right]^{\frac{1}{2}}}\right),
\end{align}
where   $\delta$, $\epsilon$, $C_p$ are defined in (\ref{hat_x}), respectively; $\alpha=\mu_1-\mu_2$ and $\mu_1$ and $\mu_2$ are the per-class mean vectors;  $o_{p}(1)$ is a variable decreasing when $p$ increasing;
\begin{align}\label{gammma}
    \Gamma \triangleq \sum_{j = 1}^{p} \frac{\alpha_j^2}{\sigma_j^2 + (2C_p\ln(1/\delta)/\epsilon)^2}
\end{align}
where $\alpha_j$ is $j$-th of $\alpha$ and $\sigma_j^2$ is defined in  (\ref{eqn_sigma}).

\end{theorem}
\begin{remark} The condition $\log p=o(n), n=o(p)$ means that $n$ grows much slower than $p$ while $\log p$ grows much slower than $n$. It is one of the common assumptions to study the high dimensional learning with low sample size \cite{2003Optimal}.

\end{remark}

Let $\tau \triangleq \ln(1/\delta)/\epsilon$ and $p \rightarrow \infty$, we derive an upper bound for $\Gamma$ defined in (\ref{gammma}).
\begin{align*}
    \Gamma < \frac{ \sum_{j = 1}^{p} \alpha_j^2}{ (2C_p\ln(1/\delta)/\epsilon)^2}  \leq \frac{1}{\tau^2} \frac{ \sum_{j = 1}^{p} |\alpha_j|^2}{\max_{a \leq n_1, b \leq n_2}\|x_1^a-x_2^b \|_1^2}  \\ \leq \frac{1}{\tau^2}\frac{ \sum_{j = 1}^{p} |\alpha_j|^2}{(\sum_{i = 1}^{p}|\alpha_i |)^2} \leq \frac{1}{\tau^2},
\end{align*}
where the second inequality dues to the definition of $C_p$. Since $C_p$ is the largest norm of data, the norm of distance between two classes should not be huger than $2C_p$. The third inequality caused by the maximum distance between two classes is no less than the distance of true means of each class with probability 1, i.e.,
$    P\left(\max_{a \leq n_1, b \leq n_2}\|x_1^a-x_2^b \|_1 \geq \sum_{j = 1}^{p} |\alpha_j|\right) \stackrel{n}{\longrightarrow} 1.
$
Thus, $\Gamma$ can be controlled by an upper-bound without $p$.

\begin{remark} \label{re0}

We can see two aspects from this theorem. First, for fixed noise with given $\epsilon,\delta$, when $p \rightarrow \infty$, denominator in the right side of (\ref{WW}) towards infinity. Thus the classification error is
\begin{align}\label{case1}
   \mathbf{W}(\hat{\delta}_n, \theta) \rightarrow 1 - \Phi(\mathcal{O} (\frac{1}{\sqrt{p}})) \rightarrow \frac{1}{2} ,
\end{align}
where $\mathcal{O}(d)$ means that it grows at
the order of $d$. According to  (\ref{case1}), it shows the
LDA classifier with high dimension performs nearly the same as random guessing, which is similar to the result in \cite{fan2008high}.

However, when we consider perturbation in  (\ref{WW}) with fixed $p$ and $n$. When $\epsilon$ and $\delta$ decrease to 0, i.e., , $ \tau  \rightarrow \infty$, which means the noise is large enough. Thus the classification error is
\begin{align}\label{case2}
\mathbf{W}(\hat{\delta}_n, \theta) \rightarrow 1 - \Phi(\mathcal{O} (\frac{1}{\tau^2})) \rightarrow \frac{1}{2},
\end{align}
which is merely random guessing without any ability to classify.
Moreover, when $p\rightarrow \infty$ and $\epsilon\rightarrow 0$ at the same time, then the classification error is
\begin{align}\label{case3}
\mathbf{W}(\hat{\delta}_n, \theta) \rightarrow 1 - \Phi(\mathcal{O} (\frac{1}{\sqrt{p \tau^4}})) \rightarrow \frac{1}{2}.
\end{align}
Compared (\ref{case3}) with (\ref{case1}), it reveals that the larger noise can speed up the rate of model degradation.
\end{remark}

So far we have analysis the impact of dimensionality on binary classifications. In the appendix, we have extended the results for multi-class classification.

From Theorem \ref{thm_1} and the Remark \ref{re0}, it shows that the larger model with high dimension leads to lower classification accuracy under DP theoretically. To trade off the classification accuracy and privacy-preserving, we use the feature selection technique to reduce the dimension of large model since Theorem \ref{thm_1} shows accuracy is about $\Gamma$ and thus about dimension $p$. Experiment of this theorem will be listed in experiment part.


\begin{figure*}[hbt]
\vspace{-0.4cm}
 \centering
\begin{subfigure}[b]{0.45\textwidth}
\includegraphics[width=\textwidth]{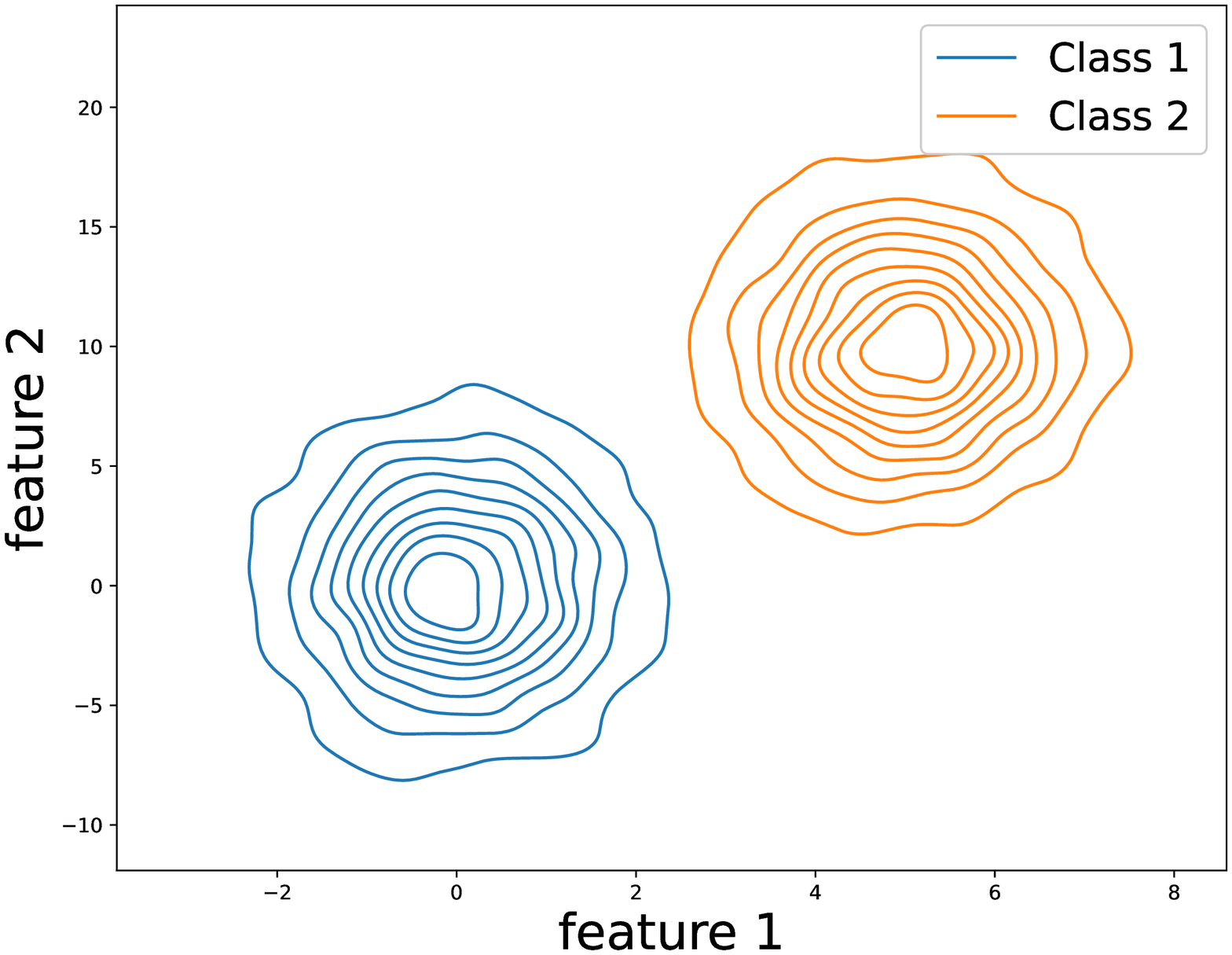}
\caption{Feature distribution before perturbation}\label{fig_vl}
\end{subfigure}
\begin{subfigure}[b]{0.45\textwidth}
\includegraphics[width=\textwidth]{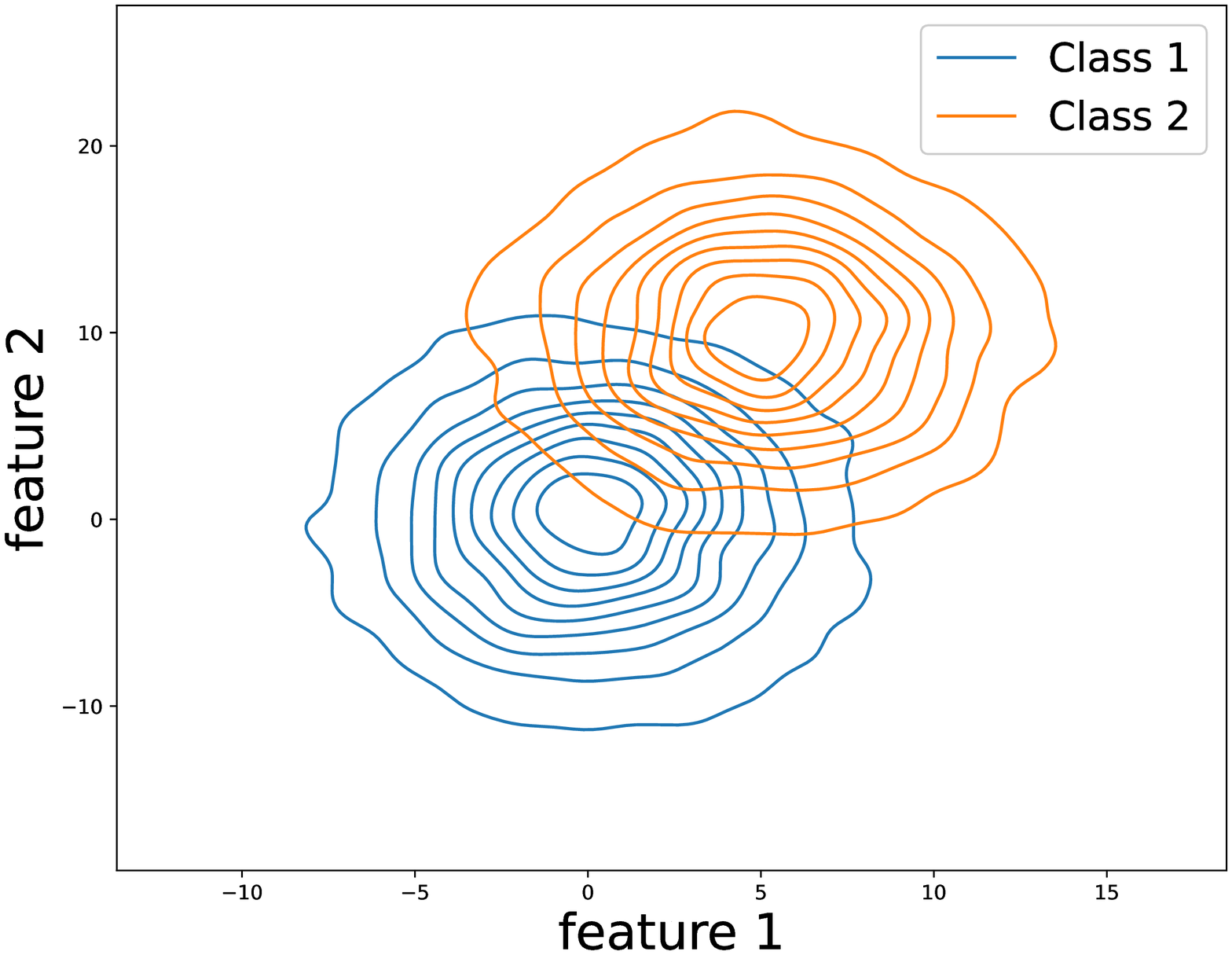}
\caption{Feature distribution after perturbation}\label{fig_vn}
\end{subfigure}
\caption{Distribution for classes in different situations. For the left figure, feature 1 of two distributions is almost no overlap which means this feature is powerful to distinguish class while feature 2 is not so powerful. For the right, both features have overlap. But feature 2 is merely over 1/2 while feature 1 is about 3/4, so feature 2 is more powerful now.}
\vspace{-0.3cm}
\end{figure*}
\subsection{Feature Selection}
In this subsection, we use filter feature selection methods, which assign a score (often a statistical measure) to each feature. One typical statistical measure is t-statistics \cite{hua2009performance}, which is defined as follows
\begin{align}
    T_{j}=\frac{\bar{x}_{1 j}-\bar{x}_{2 j}}{\sqrt{S_{1 j}^{2} / n_{1}+S_{2 j}^{2} / n_{2}}},  j=1, \cdots, p ,
\end{align}
where $\bar{x}_{k j}$ and $S_{k j}$ are defined in (\ref{mum}). After computing the values of t-statistic for each feature, we sort these values in descending order and select the important feature. Moreover, under the DP setting, we hope the feature selection result is independent of the perturbation.

When there exists a significant difference between the means of two classes, t-statistic can perform well for finding important features. However, when we add noise to the data,  the selected feature using t-statistic is susceptible to perturbation, since the formulation of t-statistic relies on sample variance. Specifically, according to the definition of $S_{k,j}$ defined in (\ref{mum}), we calculate the expectation of $\hat{\mSigma}$ as
$\mathbf{E}( \hat{\mSigma}) = \mSigma + (2C_p\ln(1/\delta) /\epsilon)^2 *\mathbf{I}_p.$
It shows that the DP budget $\epsilon$ and $\delta$ can influence the value of t-statistic. Here we also give an example to show it.




\begin{example} Consider a binary classification problem based on private GM (Definition \ref{PGMM}). The variance and mean vector set
$\mathbf{\Sigma} ={\rm diag}(1,10)$ and $\mu_1 = [0,0]$
$\mu_2 = [5,10]$ (Fig.\ref{fig_vl}), respectively.
We sample $n_1 = n_2 = 200$ for each class.

Firstly, if  we use the clean data without adding  noise in  private GM, the value of t-statistic, $D = \mu_2 - \mu_1$,  $S_1^2$ and $ S_2^2 $ are calculated as the following table.
$$
\begin{tabular}{c|c|c}
feature & 1 &2\\
\hline
   D & 5& 10 \\
   variance& 1&10 \\
    t-statistics & 50&31.62
\end{tabular}. $$
The above table shows that feature 1 has a bigger t-statistics, thus we select feature 1 if we only require one feature.

Secondly, when we add noise with  DP budget of $\epsilon = 3$ in private GM (Fig.\ref{fig_vn}), the results are presented as follows
$$
\begin{tabular}{c|c|c}
feature & 1 &2\\
\hline
     D & 5& 10 \\
   variance perturbed& 10&19 \\
    t-statistics & 15.82 & 22.94
\end{tabular}. $$
Thus feature 2 is a better result to be selected.

\end{example}
This example shows the best feature or the sort of t-statistic is not stable to perturbation, which means a small noise on data may create a new rank and it is harmful for feature selection in DP. However, numerator of t-statistics is stable since $ E(\hat{\mu}_k) = \mu_k$ regardless of perturbation (see the first row of the above tables). It suggests us to consider the following distance criterion for selecting the important feature.







\begin{definition}
Distance criterion is defined as follow:
\begin{align}
    \hat{D}_j = \hat{x}_{1j}-\hat{x}_{2j},
\end{align}
where $\hat{x}_{kj}$ is the average of class $k$, feature $j$.
\end{definition}
This is a stable criterion since $\mathbf{E}(\hat{D}) = \mu_1 -\mu_2$ whether the data has noise or not.
Next, we give a theorem to show that the proposed distance criterion can distinguish those useful features with probability one.

\begin{assumption}\label{assum}
~
\begin{enumerate}
\item Assume that distance vector $\alpha=\mu_{1}-\mu_{2}$ is sparse and without loss of generality, only the first $s$ entries are nonzero.
\item  Assume that the elements of both diagonal matrices $\mathbf{\Sigma}_{1}$ and $\boldsymbol{\Sigma}_{2}$ are bounded with upper bound $v$.
\end{enumerate}
\end{assumption}
In high dimension learning with low size data, sparsity is always a consideration \cite{grvcar2005data}. Also, variance is normal to be seen as finite, otherwise, estimation of variance will not be close to true value with a low size of data.

The following theorem describes that all important features can be selected by distance criterion.
Recall that $n=n_{1}+n_{2}$ and $n_k$ represents sample size of class $k$.
\begin{theorem} \label{thm_last}
Let $c_1 \leq n_1/n_2 \leq c_2$, $s$ be a value such that $\log (p-s)=o\left(n^{\gamma}\right)$ and $\log s = o\left(n^{\frac{1}{2}-\gamma} \beta_{n}\right)$ for some $\beta_{\mathrm{n}} \rightarrow \infty$ and
${0<\gamma<\frac{1}{3} }.$ Suppose that  $\min_{j=1,\ldots,p} \left|\alpha_{j}\right| =v n^{-\gamma} \beta_{n}$. Then under Assumption \ref{assum}, for $y =  cvn^{(\gamma-1) / 2}$ with  $c$ some positive constant, we have
\begin{align}
P\left(\min _{j \leq s}\left|\hat{D}_{j}\right| \geq y \quad \text { and } \quad \max _{j>s}\left|\hat{D}_{j}\right|<y\right) \rightarrow 1.
\end{align}
\end{theorem}

\begin{remark}
From Theorem \ref{thm_last}, we observe that the proposed distance criterion can distinguish the non-zero feature with probability one. When these important features are selected, then the dimension can be reduced. Thus, combing with Theorem \ref{thm_1},  classification accuracy and privacy-preserving can be traded off by feature selection.

\end{remark}

\subsection{DP Feature Selection Algorithm (DFS)}




Based on the proposed distance criterion, we design an integral algorithm to select the important feature under DP.

\begin{algorithm}[htb]
\caption{DP Feature Selection Algorithm}\label{alg_1}
\textbf{Input:}: [[$\mathbf{X}_{11}$],...,[$\mathbf{X}_{1n_1}$]] and [[$\mathbf{X}_{21}$],...,[$\mathbf{X}_{2n_2}$]]\\
Calculate average of features: $\hat{\mu}_1 = [a_{11},...,a_{1p}]$ and $\hat{\mu}_2 =  [a_{21},...,a_{2p}]$  \\
Calculate distance of features: $D$ = |$\hat{\mu}_1 - \hat{\mu}_2 $| \\
Rank features with distance: $X_{r}$ = [[$x_{1[1]}$,...,$x_{1[p]}$],...,[$x_{n[1]}$,...,$x_{n[p]}$]] \\
Cut the first $m$ features: $X_{c}$ = [[$x_{1[1]}$,...,$x_{1[m]}$],...,[$x_{n[1]}$,...,$x_{n[m]}$]]\\
 Calculate the maximum norm in $X_c$:  $N_{max}\triangleq \max_{i \leq n,X_i \in X_c} \|X_i \|_1$ \\
Generate noise: $n \times m$ matrix $\varepsilon$ with i.i.d. $\varepsilon_{ij} \sim \mathcal{N}(0,2N_{max}\ln(1/\delta)/\epsilon)$\\
Add noise to feature: $ \hat{X} = X_{c} + \varepsilon$\\
\textbf{Output:} feature with noise $\hat{X}$, $Label$
\end{algorithm}

Since we clip feature from $p$ to $m$ ($m<p$), $C_p$ and $p$ in  Theorem \ref{thm_1} will become smaller, thus classification error would be reduced.

\begin{remark}
This algorithm bases on our private GM. When we consider a neural network with inputs of image and text which are not vectors, we will use their latent layer of a neural network as features to utilize our algorithm.
\end{remark}


\section{Experiment}\label{sec_6}

In this section, we check our theoretical results by performing experiments on multiple common datasets, including synthetic data, RCV1 \cite{lewis2004rcv1} and CIFAR-10.
For all DP-mechanism, choose normal distribution and set $\delta = 0.0001$. Then for different data set, we choose different DP  budget of $\epsilon$   to protect the data.
\subsection{Synthetic Data}

For synthetic data, consider two high dimensional  Gaussian distributions $\mathcal{N} (\vmu_0 , \mSigma_0)$ and $\mathcal{N} (\vmu_1, \mSigma_1)$, where  $\mSigma_k = diag(a_{1_k},...,a_{p_k})$ with $p=3000$ and $a_{i_j} \sim \exp(0.1)$. In addition, $\vmu_0 = \textbf{0}\in \mathbb{R}^{p}$, $\vmu_1$ is a  $
(1-c) \delta_{0}+\frac{1}{2} c \exp (-2|x|),
$
where $\delta_0$ means equals to 0 and $c = 0.88$.
\begin{figure}[t]
\vspace{-0.42cm}
\vbox to 5.5cm{\vfill \hbox to \hsize{\hfill
\scalebox{0.265}[0.28]{\includegraphics{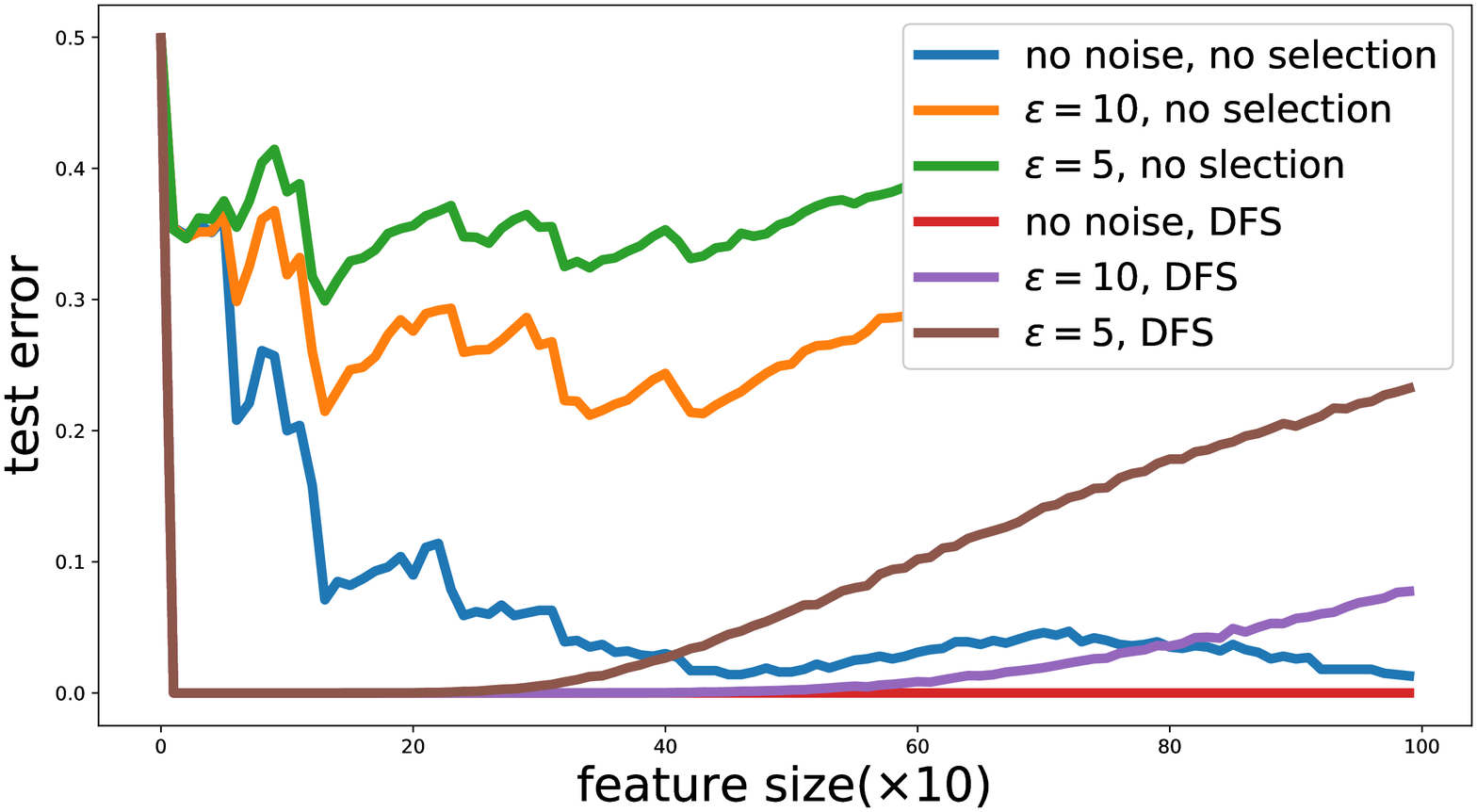}} \hfill}\vfill}
\vspace{-0.42cm}
\caption{Comparison for different DP parameters and selection. For our algorithm, only 20 features can generate a model with the highest accuracy in all DP settings while no selection model cannot reach the best with noise.}\label{fig_max}
\end{figure}
\begin{figure*}[t]
 \centering
\begin{subfigure}[b]{0.45\textwidth}
\includegraphics[width=\textwidth]{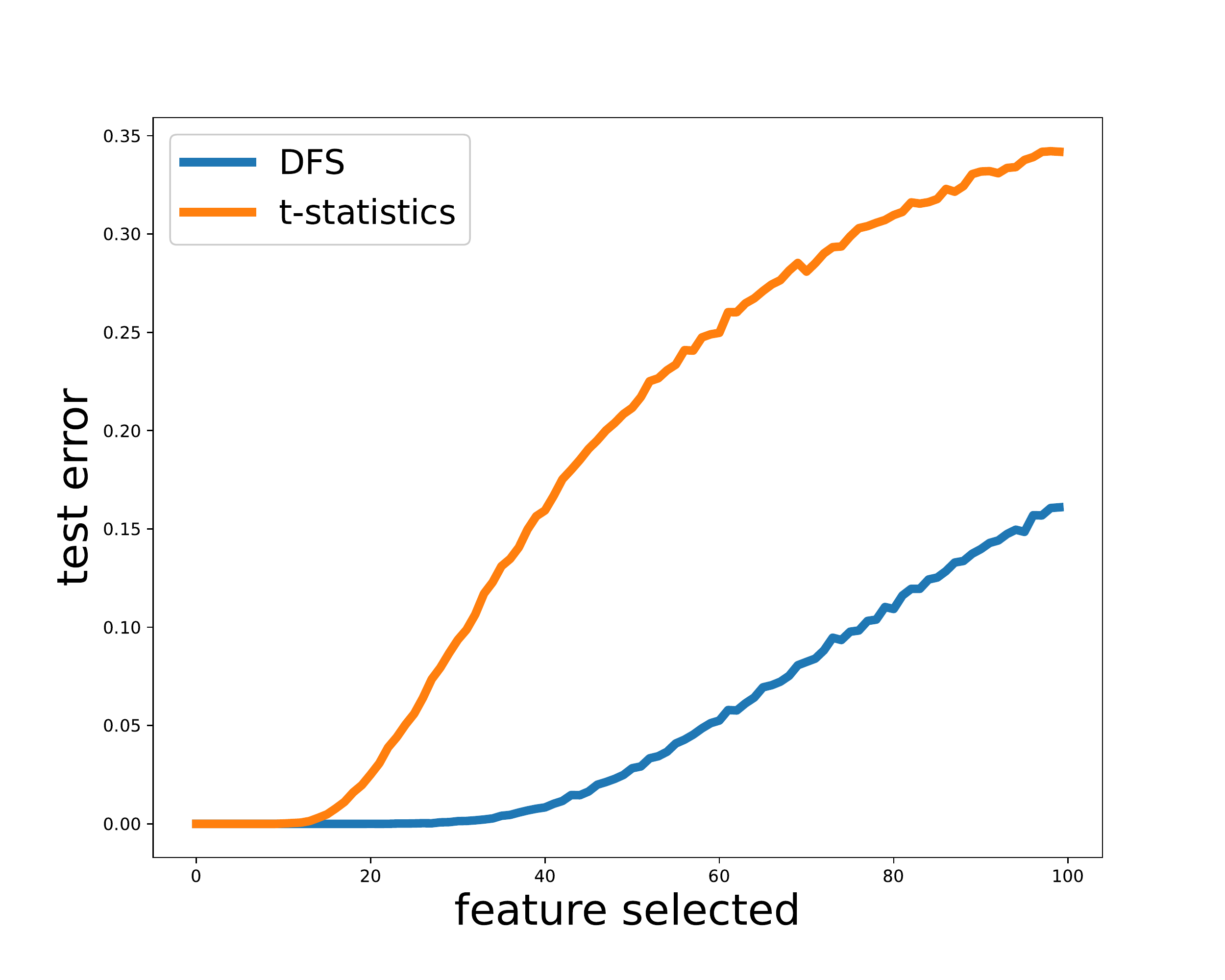}
\caption{fixed $\varepsilon = 4$ with increasing $p$}\label{fig_num}
\end{subfigure}
\begin{subfigure}[b]{0.45\textwidth}
\includegraphics[width=\textwidth]{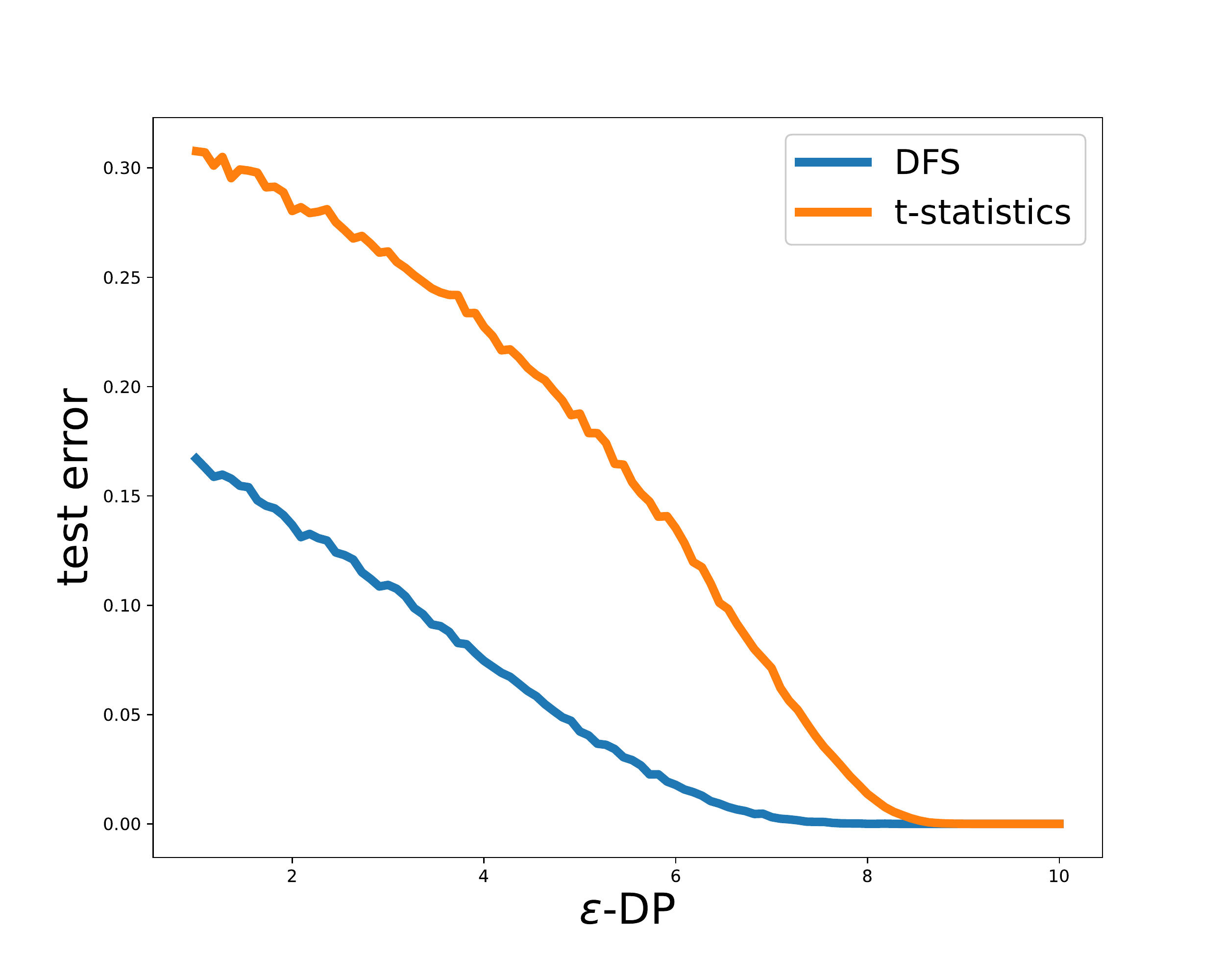}
\caption{fixed $ p = 70 $ with increasing $\varepsilon$}\label{fig_cov}
\end{subfigure}
\caption{Results for the numerical dataset. The left figure shows that our DFS maintains stable for $p < 30$ while t-statistic climbing all the time. Right shows that for fixed $p$, comparing with t-statistic, DFS obtain higher accuracy with the same DP budget $\varepsilon$. }
\end{figure*}
\begin{figure*}[t]
 \centering
\begin{subfigure}[b]{0.45\textwidth}
\includegraphics[width=\textwidth]{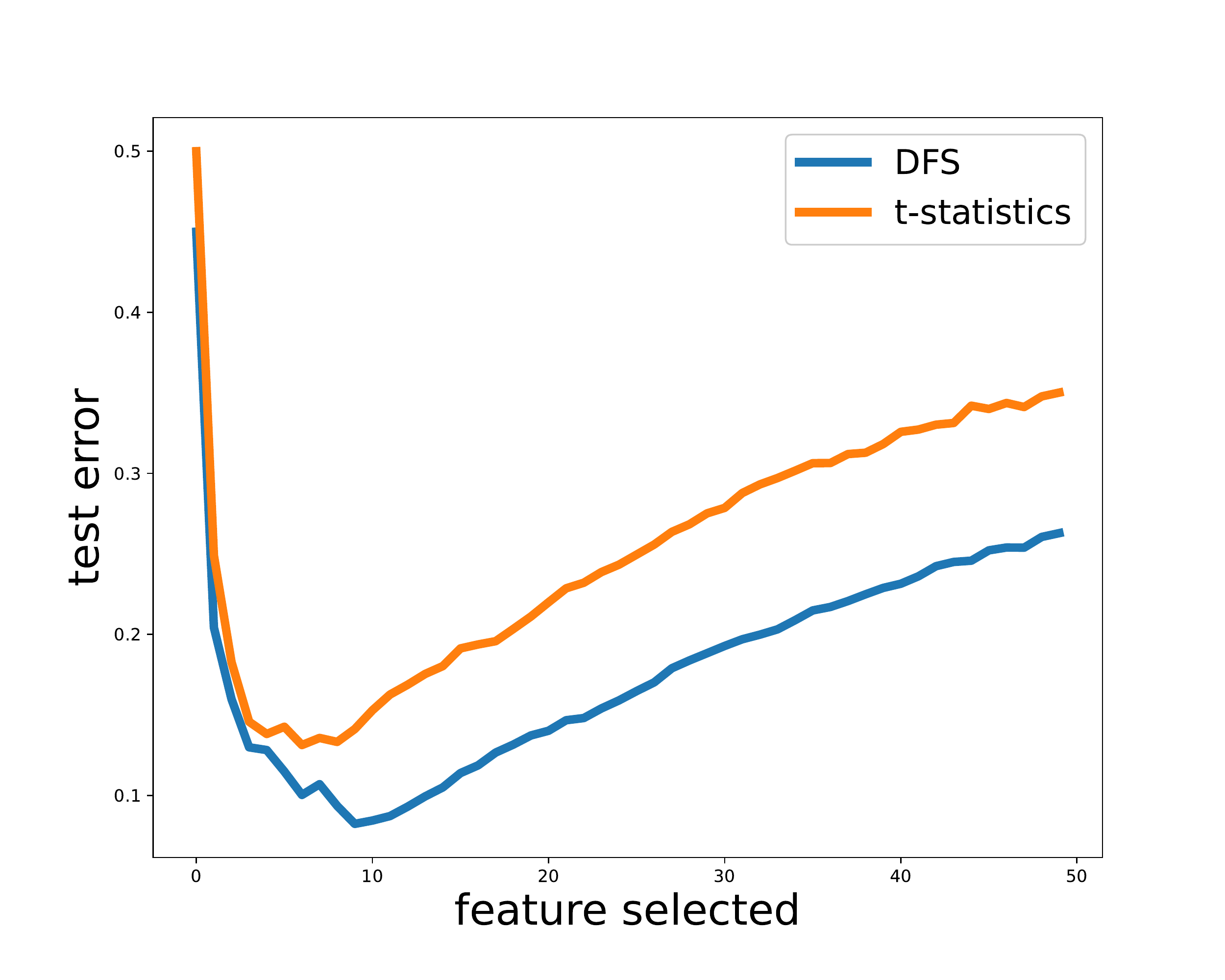}
\caption{fixed $\varepsilon = 6$ with increasing $p$}\label{fig_rcv}
\end{subfigure}
\begin{subfigure}[b]{0.45\textwidth}
\includegraphics[width=\textwidth]{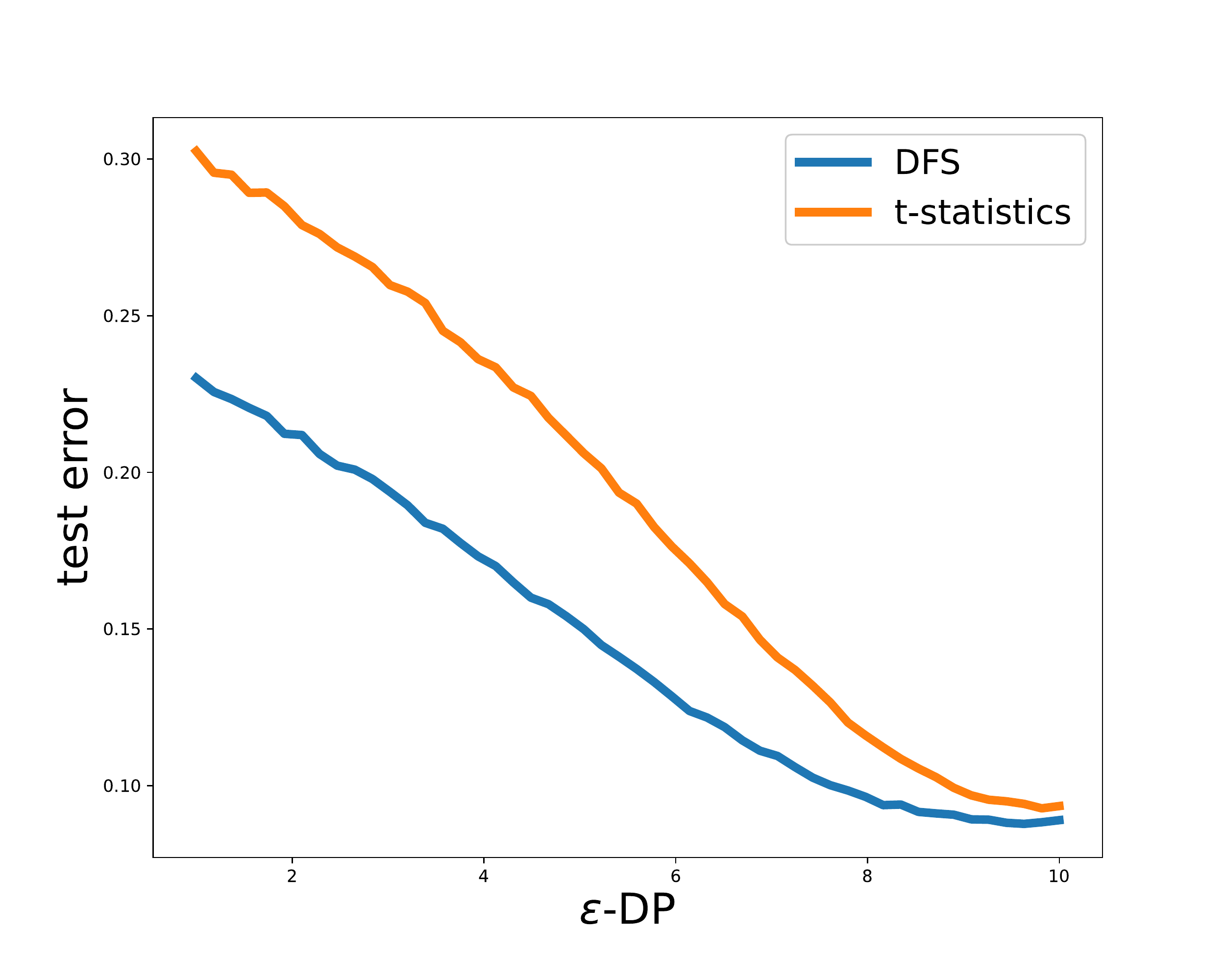}
\caption{fixed $p = 20$ with increasing $\varepsilon$}\label{fig_rcv_cov}
\end{subfigure}
\caption{Results for RCV1. The left figure shows when the dimension of feature comes to 9, both algorithm reaches the best accuracy while DFS gets 0.06 test error. Right shows for DP budget $\varepsilon $  from 1 to 6, test error of DFS is much smaller than that of t-statistic.}
\end{figure*}

Using the GM parameters above, we generate 30 training data and 200 testing data for each class. To protect privacy, we add some Gaussian noise with $\epsilon = 5$ and $\epsilon = 10$ to all data except labels basing on private GM (Definition \ref{PGMM}). The LDA classifier is used to separate these two classes. Fig. \ref{fig_max} presents the test accuracy.





In Fig.\ref{fig_max}, we observe that the proposed algorithm performs well in feature selection under different noise levels. In noiseless condition, the performance of feature selection and without feature selection is the same. However, within 20 features selected by DFS, our algorithm converges to 0 test error while feature without selection needs more than 1000 features. Moreover, DFS can select features to maintain the best accuracy for more than 500 features. Also, without feature selection, result is not smooth since some perturbing data influence the performance.




In Fig.\ref{fig_num}, a comparison of proposed algorithms with t-statistic in terms of test error has been presented. It shows that the increase of dimension leads to the decrease of performance, which is consistent with Theorem \ref{thm_1}.
 However, the curve of the proposed method is below that of t-statistic, thus the proposed method can reduce the influence of high dimension. Since the larger $\varepsilon$ means the smaller noise, Fig. \ref{fig_cov} shows the test error decreasing with the increasing the $\varepsilon$ while the proposed method can reduce the test error much more.
In addition, in both figures, our curves maintain parallel for a long time, which means the proposed method is more resilient to dimension increasing and noise accumulation.

\subsection{RCV1}

RCV1 dataset  contains over 800,000 manually categorized news-wire stories labeled with news category, embedding into 47236 features. So we can regard it as features after extractor. Then we set it into a binary classification problem by choosing random 2 classes and draw 40 data each for training and 200 for the test.

For Fig.\ref{fig_rcv} and Fig.\ref{fig_rcv_cov}, it also shows the test error of DFS is smaller than that of t-statistic.  It is convinced that our algorithm can outperform traditional selection methods by t-statistic in the DP condition. Thus the proposed method provides a solution for the issue in the introduction.

\subsection{CIRAR-10}

Recall our original problem in the introduction that ResNet50 draws back more due to noise accumulation, our selection rule helps to reduce this tendency. (Algorithm for multiple classes and details in this experiment is listed in the appendix.)

In this experiment, we use the last but one layer data of ResNet training by DP-SGD setting $\epsilon = 5$ to represent the input data in our algorithm. For a fair comparison, since  ResNet18 has 512 features, we select 512 features from 2048 in ResNet50. Then we use multi-layer perceptrons (MLP) to train it with SGD without noise.




In Table \ref{accu}, beyond that our algorithm can raise accuracy for ResNet50, we also show that our method is better than the classic approaches which consider variance like t-statistics.

\begin{table}[h]
\centering
\begin{tabular}{@{} l rrr@{}}
		& \multicolumn{3}{c}{CIFAR-10}\\[-0.25em]
		\cmidrule(l{5pt}r{0pt}){2-4}
		Model & Min & Max & Median\\
		\toprule
		ResNet50 & 75.5 & 79.2 &  77.0 \\
		ResNet18 & 83.6 & 85.3 & 84.5 \\
		ResNet50+t-statistic	& 78.4 & 81.3 & 80.1\\
		ResNet50+DFS & \bf{84.8} & \bf{86.4} & \bf{85.7} \\
		\toprule
	\end{tabular}
\caption{Result for features on CIRAR-10 with ResNet18/50 under DP condition. We select 512 features from ResNet50 by DFS, then we see  ResNet50 performs better than ResNet18. But the test accuracy of ResNet50 by t-statistics  is less than ResNet18.}
\label{accu}
\end{table}
\section{Conclusions}
This paper has studied the phenomenon that larger models causes lower classification accuracy under DP.
 To illustrate our idea, we have considered a simple Gaussian model for analysis. When noise or dimension tends to infinity, the classifier using all features performs nearly the same as random guessing. Hence it is necessary to find a  robust distance criterion to reduce the dimension of the model. Theoretically, we have proved that the important features can be selected with probability one.
 Finally,  we have proposed a DFS algorithm to trade off the classification accuracy and privacy-preserving. Simulation  reveals that the proposed DFS algorithm enjoys better performance on real data.
A future direction is to analyze the impact of the dimensionality under DP in a more realistic model.



{
    \small
    \bibliographystyle{ieee_fullname}
    \bibliography{macros,main}
}


\setcounter{page}{1}

\twocolumn[
\centering
\Large
\textbf{Towards Understanding the Impact of Model Size on Differential Private Classification} \\
\vspace{0.5em}Supplementary Material \\
\vspace{1.0em}
] 
\appendix

\section{Appendix}

\subsection{DP-SGD CNN for MNIST}
In this experiment, we use CNN for about 26k parameters and 52k parameters by widening layers to train on MNIST using DP-SGD. The result is similar comparing with the introduction that models with more parameters perform worse.

\begin{figure}[htb]
\vspace{-0.5cm}
\vbox to 5.5cm{\vfill \hbox to \hsize{\hfill
\scalebox{0.3}[0.305]{\includegraphics{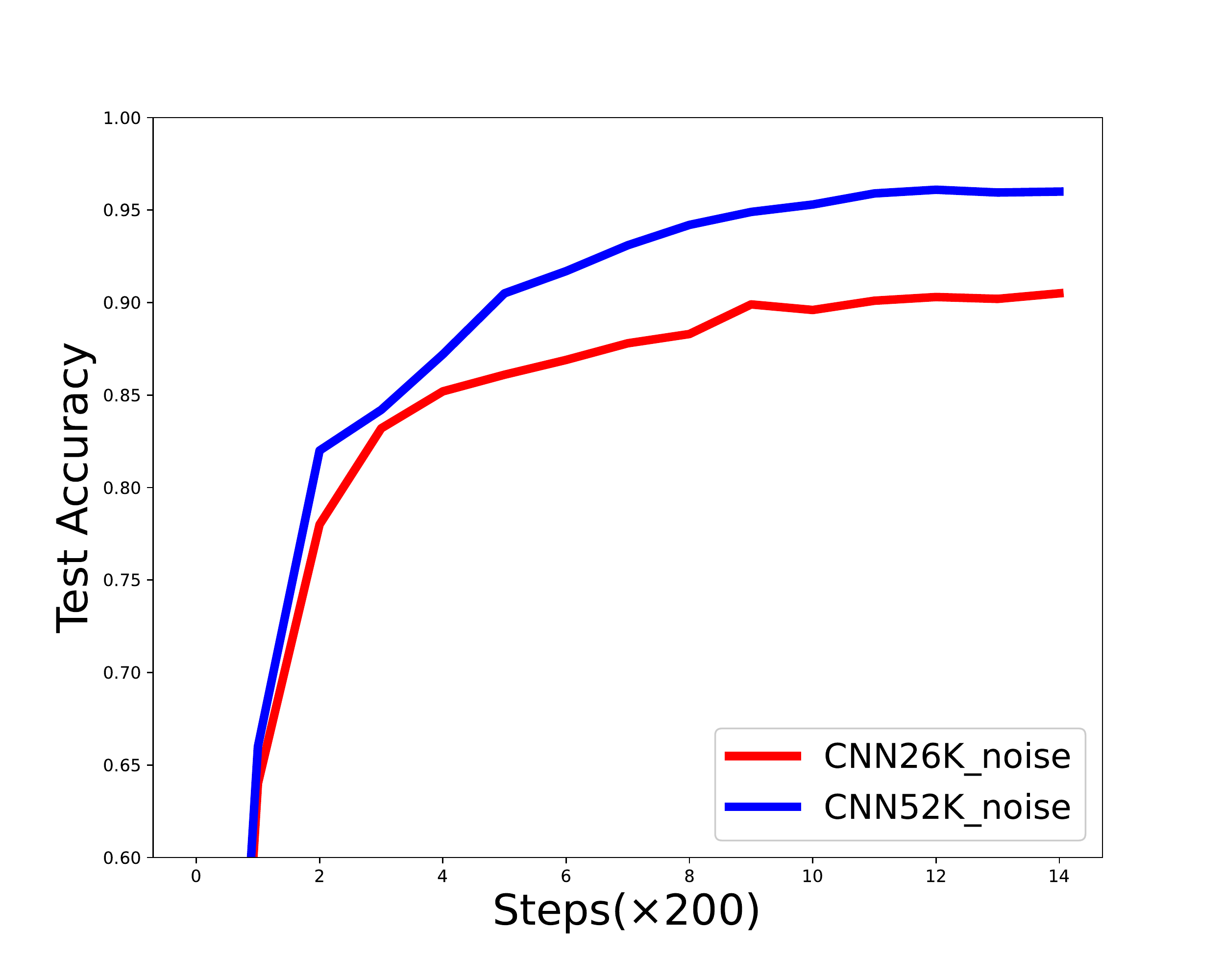}} \hfill}\vfill}
\caption{Accuracy for CNN using 52k parameters after DP-SGD perform 5\% lower than 26k parameters.}
\label{Fig_MN}
\end{figure}

This also shows that both width and depth of neural network would influence accuracy, leading to problem towards dimension.

\subsection{Proof of Theorem 6}
Before proof, we agree that character with hat is observation in this proof and truth value without hat.
\begin{proof}
First, since for normal distribution, if we have $x \sim \mathcal{N}(a,b)$ and $ y \sim \mathcal{N}(c,d) $, assume $x$ and $y$ are independent, then
\begin{align*}
    x+ y \sim \mathcal{N}(a+c,b+d)
\end{align*}
Thus data with perturbation can regard as a new data set. We will show result with new $\mSigma \triangleq \mSigma + \sigma^2 * \mathbf{I}_p$.
Also, we define $ \epsilon_{i j}$ is the bias for data i and feature j from true means, $S_j$ is average estimation variance for feature j and $\hat{\epsilon}_{k j }$ is the average bias for class k and feature j.

For estimation $\hat{\Sigma}$, we have following inequality:
\begin{align*}
P\left(\max _{j=1, \cdots, p}\left|S_{j}^{2}-\sigma_{j}^{2}\right|>\varepsilon\right) \leq \sum_{j=1}^{p} P\left(\left|S_j^{2}-\sigma_{j}^{2}\right|>\varepsilon\right)  \\ \leq \sum_{j=1}^{p} P\left(\left|\sum_{i=1}^{n}\left(\epsilon_{i j}^{2}-\sigma_{j}^{2}\right)\right|>n \varepsilon \right) \equiv I_1.
\end{align*}
It follows from Bernstein's inequality that
\begin{align*}
P\left(\left|\sum_{i=1}^{n_{k}}\left(\epsilon_{k i j}^{2}-\sigma_{k j}^{2}\right)\right|>n \varepsilon \right) \leq 2  \exp \left\{-\frac{c}{2} \frac{n^{2} \varepsilon^{2}}{\sum_{j = 1}^{p}\sigma^2_j}\right\},
\end{align*}
where c is the parameter for Bernstein's inequality.

Since $\log p=o(n)$, when $p \rightarrow \infty $, $n \rightarrow \infty$. So $I_{1}=o_p(1)$. Thus $P\left(\max _{j=1, \cdots, p}\left|S_{j}^{2}-\sigma_{j}^{2}\right|>\varepsilon\right) \stackrel{P}{\rightarrow} 0 $. So $\hat{\mSigma} = (1+o_p(1))\mSigma$.

Then we back to definition of classification error, since we assume $\Sigma$ is already a diagonal matrix, after simplification, it can be written in form $W(\hat{\delta}, \theta) = 1- \Phi(\Psi)$ where \begin{align*}
    \Psi \geq\frac{\left(\mu_{1}-\hat{\mu}\right) \prime \hat{\mathbf{\Sigma}}^{-1}\left(\hat{\mu}_{1}-\hat{\mu}_{2}\right)}{\sqrt{\left(\hat{\mu}_{1}-\hat{\mu}_{2}\right) \prime \mathbf{\Sigma}^{-1} \left(\hat{\mu}_{1}-\hat{\mu}_{2}\right)}}(1+o_p(1)),
\end{align*}

Since $\hat{\mSigma} = (1+o_p(1))\mSigma$, $\hat{\mathbf{\Sigma}}^{-1} = (1+o_p(1))\mSigma^{-1} $.

For numerator, we have
\begin{align*}
& \left(\mu_{1}-\hat{\mu}\right) \prime \hat{\mathbf{\Sigma}}^{-1}\left(\hat{\mu}_{1}-\hat{\mu}_{2}\right) =\frac{1}{2} \alpha \prime \hat{\mathbf{\Sigma}}^{-1} \alpha \\ & - \frac{1}{2}\left(1+o_{P}(1)\right) \sum \frac{\hat{\epsilon}_{1 j}^{2}}  {\sigma_{j}^{2}}+\frac{1}{2}\left(1+o_{P}(1)\right) \sum \frac{\hat{\epsilon}_{2 j}^{2}}{\sigma_{j}^{2}}.
\end{align*}
The third term is in the same form with fourth, so they vanish.

For denominator, it is complicated, but in the same way.\\
\begin{align*}
\left(\hat{\mu}_{1}-\hat{\mu}_{2}\right) \prime \mathbf{\Sigma}^{-1}\left(\hat{\mu}_{1}-\hat{\mu}_{2}\right) &=\alpha \prime \mathbf{\Sigma}^{-1} \alpha+2 \sum \alpha_{j} \frac{\hat{\epsilon}_{1 j} \hat{\epsilon}_{2 j}}{\sigma_{j}^{2}}\\ &+\sum \frac{\left(\hat{\epsilon}_{1 j}-\hat{\epsilon}_{2 j}\right)^{2}}{\sigma_{j}^{2}} \\
&=\alpha \prime \mathbf{\Sigma}^{-1} \alpha+2 \sum \frac{\alpha_{j}}{\sigma_{j}^{2}}\left(\hat{\epsilon}_{1 j}-\hat{\epsilon}_{2 j}\right) \\ & + \sum \frac{\left(\hat{\epsilon}_{1 j}-\hat{\epsilon}_{2 j}\right)^{2}}{\sigma_{j}^{2}}
\end{align*}

The third term is with distribution $\hat{\epsilon}_{1 j} - \hat{\epsilon}_{2 j} \sim \mathcal{N}(0,(4/n) \sigma_j^2)$. In term, it need to divide $\sigma_j^2$, so it converges to $4/n$.
\begin{align*}
    \sum \frac{\left(\hat{\epsilon}_{1 j}-\hat{\epsilon}_{2 j}\right)^{2}}{\sigma_{j}^{2}} \stackrel{P}{\rightarrow} \frac{4p}{n}
\end{align*}

Then the second term is the same.  $\frac{\alpha_{j}}{\sigma_{j}^{2}}\left(\hat{\epsilon}_{1 j}-\hat{\epsilon}_{2 j}\right) \sim N(0,(4/n)\alpha_j \sigma_j^{-2} \alpha_j)$. Then as variance is $o_p(1)$, so the whole term is in the order of $o_p(1)\alpha \prime \hat{ \mSigma }^{-1} \alpha$.

Finally, together above result, we can complete our proof.
\end{proof}

\subsection{Multi-class  private GM}
In this section, we extend the binary private GM model to multi-class model. Let us define the multi-class private GM as follows.


\begin{definition}\label{label}(Multi-class Private GM)
For $K$-class private GM, let $\mu_k \in \mathbb{R}^{p}$, $k=\{1, ..., K\}$, be the per-class mean vector and
\begin{align}
    \mathbf{\Sigma} \triangleq {\rm diag} (\sigma_1^2, ..., \sigma_p^2)
\end{align}
be the variance parameter. $(\epsilon, \delta )-$private Gaussian mixture model is defined by the following distribution over $(\hat{x}_k,k)\in \mathbb{R}^{p}\times\{1, ..., K\}$:  First, draw a label $k$ from $\{1,...,K\}$ uniformly at random, then sample the data point $x_k\in \R^p$ from $\mathcal{N}(\mu_k, \mathbf{\Sigma})$. Then we get a non-private dataset $\{x_{k}^i,k\}$, $k=\{1,...,K\}$, $i=1,\ldots,n_k$ where $n_k$ is sample size of label $k$. Finally, according Gaussian mechanism to obtain private dataset $\{\hat{x}_k^i,k\}$, where
\begin{align}
  \hat{x}_k^i = x_k^i +
  2C_p\ln(1/ \delta)/ \epsilon \cdot (\eta_1, ..., \eta_p),
\end{align}
where $\eta_i$ are i.i.d variables $\eta_i \sim \mathcal{N}(0,1)$ and $C_p \triangleq \max_{k \in \{1,...,K\}, i \leq n_k} \| x_k^i \|_1$ is a constant depending on dimension $p$.
\end{definition}

Similar to definition for private GM, estimation of parameters above is  \begin{align} \label{calcu}
    \hat{\mu}_{k}=  \frac{1} { n_{k}}\sum_{i=1}^{n_{k}} \hat{x}_{k}^i , k \in \{1,...,K\}, \\ \hat{\mathbf{\Sigma}}=\operatorname{diag}\left\{\frac{\sum_{k = 1}^K S_{k l}^2}{K}, l=1, ..., p\right\},
\end{align}
where $ S_{k l}^{2}=\frac{1}{\left(n_{k}-1\right)}\sum_{i=1}^{n_{k}}\left(\hat{x}_{k l}^i-\bar{x}_{k l}\right)^{2}  $ is the sample variance of the $l$-th feature in class $k$ and $\bar{x}_{k l}=\frac{1} { n_{k}}\sum_{i=1}^{n_{k}} \hat{x}_{k l}^i$.

For binary model, a point $x$ is labeled 1 if the fisher classifier satisfies
\begin{align*}
    (x - (\hat{\mu}_1+\hat{\mu}_2)/2)\prime \hat{\mSigma}^{-1}(\hat{\mu}_1-\hat{\mu}_2) > 0.
\end{align*}
Rewriting the form above, it is equivalent to \begin{align}
    (x - (\hat{\mu}_1+\hat{\mu}_2)/2)\prime \hat{\mSigma}^{-1}\hat{\mu}_1 > (x - (\hat{\mu}_1+\hat{\mu}_2)/2)\prime \hat{\mSigma}^{-1}\hat{\mu}_2 \\ \Leftrightarrow
    (x - \hat{\mu}_1/2)\prime \hat{\mSigma}^{-1}\hat{\mu}_1 > (x - \hat{\mu}_2/2)\prime \hat{\mSigma}^{-1}\hat{\mu}_2.
\end{align} \label{trans}
Thus it can be defined as the maximum of $(x - \hat{\mu}_i/2)\prime \hat{\mSigma}^{-1}\hat{\mu}_i (i \in \{1,2\})$. So we define multi-class LDA as follow.

\begin{definition}(Multi-class LDA classifier) \cite{li2019multiclass}
For a K-class classification with means $\{\mu_1, ..., \mu_k\}$ and the same variance $\mSigma$, a LDA classifier is defined for a new point x as \begin{align*}
    \hat{Y} = \argmax_{i \in {1, ..., K}} (x - \hat{\mu}_i/2)\prime \hat{\mSigma}^{-1}\hat{\mu}_i
\end{align*}
where $\hat{Y}$ is predicted label.
\end{definition}


The next theorem is about classification error for any classes more than 1. In this theorem, we prove that general classification error of multi-class LDA classifier for a point from class $m$, which defined as \begin{align*}
    \mathbf{W}(\hat{\delta}_n,\theta) \triangleq P(\hat{\delta}_n(x) \neq m|\hat{x}_{k}^i,k \in \{1,...,K\},i=1,\ldots,n_k),
\end{align*}
will increase with perturbation and dimension accumulating. In simple, we rewrite condition $\{\hat{x}_{k}^i,k \in \{1,...,K\},i=1,\ldots,n_k\}$ as $\{X\}$.

\begin{theorem}
for K-class problem with sample size $n_i = n_j (i,j \in \{1, ..., k\})$, $\sum_i n_i = n$ with dimension $log(p)$ = $o(n)$ and $n$ = $o(p)$ classification error for class $m$ is bounded as
\begin{align*}
    \mathbf{W}(\hat{\delta}_n, \theta) \leq 1- \prod \limits_{i \neq m, 1 \leq i \leq K} \Phi \left( \frac{ \left(1+o_{p}(1)\right)\Gamma_i}{2\left[\frac{4Kp}{2n}+ \left(1+o_{p}(1)  \right)\Gamma_i\right]^{\frac{1}{2}}}\right)
\end{align*}
where \begin{align*}
    \Gamma_i \triangleq \sum_{j = 1}^{p} \frac{\alpha_{ij}^2}{\sigma_j^2 + (2C_p\ln(1/\delta)/\epsilon)^2}
\end{align*}
and $\alpha_{ij}$ is the j-th of $\hat{\mu}_m - \hat{\mu}_i$, $\sigma_j$ in definition\ref{label} and \ref{calcu}.
\end{theorem}

\begin{proof}
Proof by contradiction, if there exists another class $m_0$ which point x should be in $m_0$ in binary LDA classification with m, then we have:
\begin{align*}
    (x - (\hat{\mu}_{m}+\hat{\mu}_{m_0})/2)\prime \hat{\mSigma}^{-1}(\hat{\mu}_{m}-\hat{\mu}_{m_0}) < 0,
\end{align*}
which equivalents to inequality\begin{align*}
    (x - \hat{\mu}_{m_0}/2)\prime \hat{\mSigma}^{-1}\hat{\mu}_{m_0} > (x - \hat{\mu}_m/2)\prime \hat{\mSigma}^{-1}\hat{\mu}_m
\end{align*}
with the same transformation of (\ref{trans}). Since class m is not the largest $ (x - \hat{\mu}_{i}/2)\prime \hat{\mSigma}^{-1}\hat{\mu}_{i}$, this situation is not legal for a correct classification.

Thus it can be represented by K-1 events,
\begin{align*}
    \kappa_i = \{(x - (\hat{\mu}_{m}+\hat{\mu}_{i})/2)\prime \hat{\mSigma}^{-1}(\hat{\mu}_{m}-\hat{\mu}_{i})>0\}.
\end{align*}
In simplified form, $1-\mathbf{W}(\hat{\delta}_n, \theta)$ can be written as
\begin{align*}
  P(\delta_n(x) = m| X)  &= P(\kappa_1, ..., \kappa_K|X) \\ &= P(\kappa_1|X) P(\kappa_2|\kappa_1,X)\\ &...P(\kappa_{m-1}|\kappa_1,...\kappa_{m-2},X) \\ & *P(\kappa_{m+1}|\kappa_1,...\kappa_{m-1},X) \\ & ...P(\kappa_{K}|\kappa_1,...\kappa_{m-1},\kappa_{m+1}...,\kappa_{K-1},X).
\end{align*}
Since $\mu_k$ is drawn independently, conditions in probability can be removed. Thus \begin{align*}
    P(\delta_n(x) = m| X) = \prod \limits_{1\leq i \leq K, i \neq m}P(\kappa_i|X).
\end{align*}
Together with binary LDA classifier in Theorem 6, we can conclude that
\begin{align*}
    & P(\delta_n(x) = m| X) \geq \\& \prod \limits_{i \neq m, 1 \leq i \leq K} \Phi \left( \frac{ \left(1+o_{p}(1)\right)\Gamma_i}{2\left[\frac{4Kp}{2n}+ \left(1+o_{p}(1)  \right)\Gamma_i\right]^{\frac{1}{2}}}\right).
\end{align*}
It completes the proof.
\end{proof}

This theorem shows in multi-class problem, our remark holds that with dimension and perturbation increasing, performance of classifier will decrease and drops to $\frac{1}{2^{k-1}}$.

\subsection{Lemma 2}\cite{cao2007moderate}
Let $n=n_{1}+n_{2} .$ Assume that there exist $0<c_{1} \leq c_{2}<1$ such that $c_{1} \leq n_{1} / n_{2} \leq$ $c_{2} .$ Let ${\tilde{T}}_{j}=T_{j}-\frac{\mu_{j 1}-\mu_{j 2}}{\sqrt{s_{1 j}^{2} / n_{1}+S_{1 j}^{2} / n_{2}}}$. Then for any $x \equiv x\left(n_{1}, n_{2}\right)$ satisfying $x \rightarrow \infty$ and $x=o\left(n^{1 / 2}\right)$,
$$
\log P\left(\tilde{T}_{j} \geq x\right) \sim-x^{2} / 2, \quad \text { as } \quad n_{1}, n_{2} \rightarrow \infty
$$
If in addition, if we have $E\left|Y_{1 i j}\right|^{3}<\infty$ and $E\left|Y_{2 i j}\right|^{3}<\infty$, then
$$
\frac{P\left(\tilde{T}_{j} \geq x\right)}{1-\Phi(x)}=1+O(1)(1+x)^{3} n^{-1 / 2} d^{3}, \quad \text { for } \quad 0 \leq x \leq n^{1 / 6} / d
$$
where \\ $d=\left(E\left|X_{1 i j}\right|^{3}+E\left|X_{2 i j}\right|^{3}\right) /\left(\operatorname{var}\left(X_{1 i j}\right)+\operatorname{var}\left(X_{2 i j}\right)\right)^{3 / 2}$ and $O(1)$ is a finite constant depending only on $c_{1}$ and $c_{2} $.

\subsection{Proof of Theorem 8}
\begin{proof}
First, since we consider Gaussian distribution, so lemma 2 is always tenable in below proof, we will use it directly.

Second, take into two parts.
a) First, we check probability $P(max_{j>s}|D_j|>x)$. For any probability, it is clear that $$
P\left(\max _{j>s}\left|D_{j}\right|>x\right) \leq \sum_{j=s+1}^{p} P\left(\left|D_{j}\right| \geq x\right).
$$
With lemma 2 and the max variance bounded after normalization, we can infer that $$
P(\max _{j>s}|D_j|>v\frac{2}{\sqrt{n}}x) \leq (1-\Phi(x))\left(1+C(1+x)^{3} n^{-1 / 2} d^{3}\right)
$$ with $d=\left(E\left|Y_{1 i j}\right|^{3}+E\left|Y_{2 i j}\right|^{3}\right) /\left(\sigma_{1 j}^{2}+\sigma_{2 j}^{2}\right)^{3 / 2}.$
\\Since $|T_j|$ obtain following inequality:$$
|T_j| = \frac{|D_j|}{\sqrt{s_{1j}^2/n_1+s_{2j}^2/n_2}} \geq \frac{|D_j|}{v}\sqrt{\frac{n_1n_2}{n}}\geq \frac{\sqrt{n}}{2}\frac{|D_j|}{v}.
$$
Also, with normal distribution
$$
1-\Phi(x) = \frac{1}{\sqrt{2\pi}}\int_{x}^{\infty} e^{-x^2/2} dx < \frac{1}{\sqrt{2\pi}}\int_{x}^{\infty} e^{-xy/2} dy,
$$we can give that
$$
1-\Phi(x) \leq \frac{1}{\sqrt{2 \pi}} \frac{1}{x} e^{-x^{2} / 2}
$$
This together with the symmetry of $D_{j}$ gives
$$
P(|D_j|>v\frac{2}{\sqrt{n}}x) \leq 2 \frac{1}{\sqrt{2 \pi}} \frac{1}{x} e^{-x^{2} / 2}\left(1+C(1+x)^{3} n^{-1 / 2} d^{3}\right) .
$$
Combining the above inequality, we have
$$
\sum_{j>s} P(|D_j|>v\frac{2}{\sqrt{n}}x) \leq(p-s) \frac{2}{\sqrt{2 \pi}} \frac{1}{x} e^{-x^{2} / 2}\left(1+C(1+x)^{3} n^{-1 / 2} d^{3}\right).
$$
Since $\log (p-s)=o\left(n^{\gamma}\right)$ with $0<\gamma<1/3$, if we let $x = c n^{\gamma / 2}$, that is $y = cvn^{(\gamma-1)/2},$ then
$$
\sum_{j>s} P\left(\left|D_{j}\right| \geq y\right) = n^{\gamma-1/2}
.$$
So we can draw that
$$
\sum_{j>s} P\left(\left|D_{j}\right| \geq y\right) \rightarrow 0.
$$
This equality yields
$$
P\left(\max _{j>s}\left|D_{j}\right| \geq y\right) \rightarrow 0.
$$
b) Then we consider $P(\min_{j\leq s}|D_j|\leq y) $. Notice that when $j \leq s$, $\alpha_j = \mu_{1j}-\mu_{2j} \neq 0$. So also with lemma 2, we define $\tilde{D}_j = D_j - \alpha_j$, it is same like a)
\begin{equation}
 P(\max_{j \leq s}|\tilde{D}_j|>v\frac{2}{\sqrt{n}}x) \rightarrow 0.
 \label{eq_1}
\end{equation}
For addition, there is an inequality
$$
y > \min_{j \leq s}|D_j| = \min_{j \leq s} |\tilde{D_j}+\alpha_j| \geq \min_{j \leq s} |\alpha_j|-\max_{j \leq s} |\tilde{D}_j|.
$$
So in probability
$$
P\left(\min _{j \leq s}\left|D_{j}\right| \leq y\right) \leq P\left(\max _{j > s}\left|\tilde{D}_{j}\right| \geq \min _{j \leq s}\left|\alpha_{j}\right|-y\right).
$$
Then with all assumption above and some $\beta_n \rightarrow \infty$
$$\min _{j \leq s}\left|\alpha_{j}\right|-y = vn^{-\gamma}\beta_n - 2cvn^{(\gamma-1)/2} \geq y.
$$
Together with (\ref{eq_1}), b) is established. Combination two parts complete the theorem.
\end{proof}

\subsection{Means of $\mu_1$ in toy experiment}

Fig. \ref{fig_cov_dis} is a bar figure of  our $\vmu_1$ in toy experiment. We can see most of the features are sparse.

\begin{figure}[htb]
\vbox to 5.5cm{\vfill \hbox to \hsize{\hfill
\scalebox{0.3}[0.305]{\includegraphics{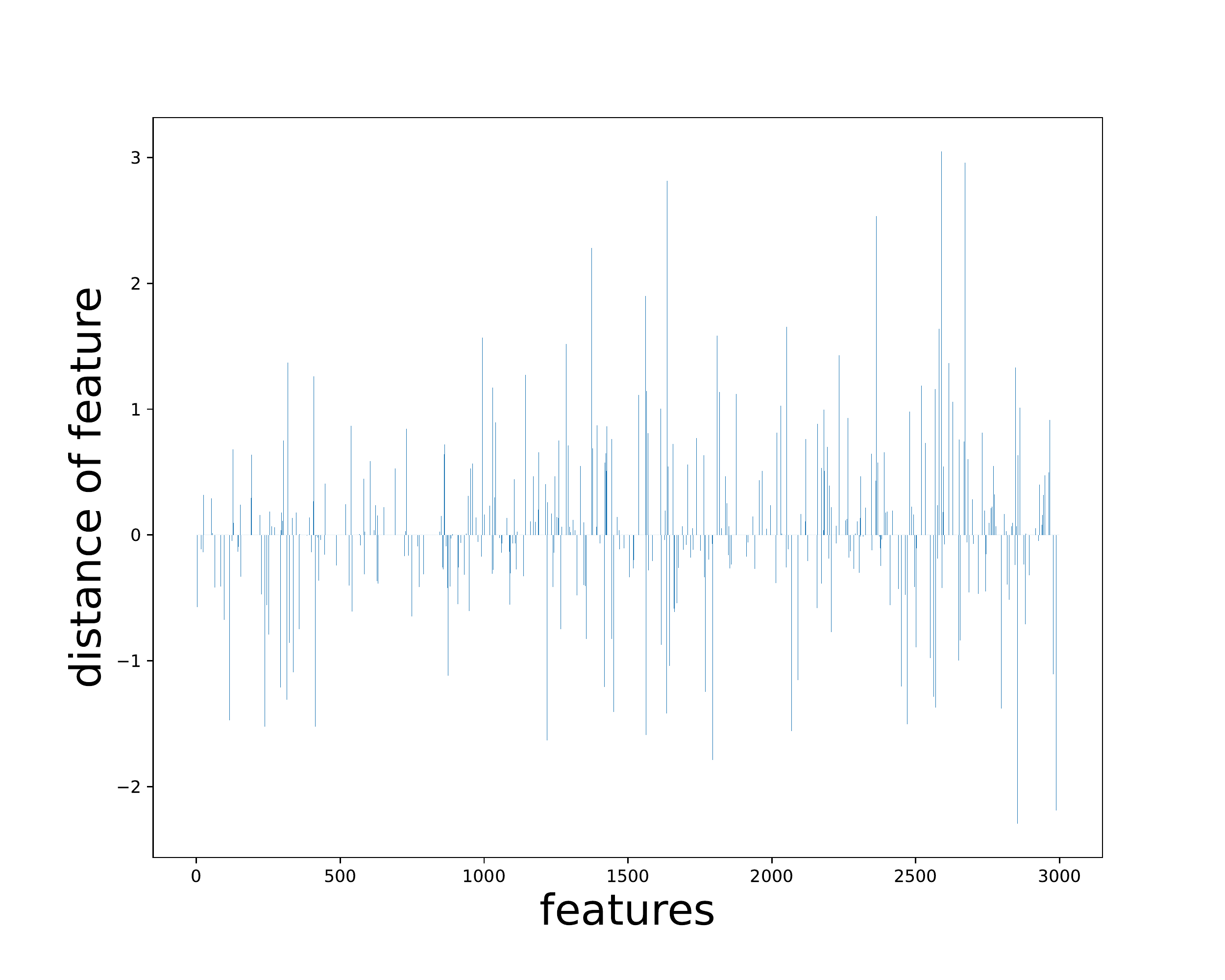}} \hfill}\vfill}
\caption{distance of different classes for all features in synthetic data}\label{fig_cov_dis}
\end{figure}

\subsection{Multiple class criterion}
For LDA's classifier, we consider in binary classification. But our approach can be generalized to multiple classification. We list changed algorithm in CIFAR-10 part.

\begin{algorithm}[htb]
\caption{DP Feature Release Algorithm with $K$ classes}
\textbf{Input:}  [[$\mathbf{X}_{11}$],...,[$\mathbf{X}_{1n_1}$]] to [[$\mathbf{X}_{K1}$],...,[$\mathbf{X}_{Kn_K}$]]\\
Calculate average of features: $\hat{\mu}_1 = [a_{11},...,a_{1p}]$ to $\hat{\mu}_K =  [a_{K1},...,a_{Kp}]$  \\
Calculate max distance of features: $D_j$ = $\max_{c,q\leq K}$|$\hat{\mu}_{qj} - \hat{\mu}_{cj} $| \\
Rank features with distance: $X_{r}$ = [[$x_{1[1]}$,...,$x_{1[p]}$],...,[$x_{n[1]}$,...,$x_{n[p]}$]] \\
Cut the first $m$ features: $X_{c}$ = [[$x_{1[1]}$,...,$x_{1[m]}$],...,[$x_{n[1]}$,...,$x_{n[m]}$]]\\
 Calculate the maximum norm in $X_c$:  $N_{max}\triangleq \max_{i \leq n,X_i \in X_c} \|X_i \|_1$ \\
Generate noise: $n \times m$ matrix $\varepsilon$ with i.i.d. $\varepsilon_{ij} \sim \mathcal{N}(0,2N_{max}\ln(1/\delta)/\epsilon)$\\
Add noise to feature: $ \hat{X} = X_{c} + \varepsilon$\\
\textbf{Output:} feature with noise $\hat{X}$, $Label$
\label{multi}
\end{algorithm}

\newpage

\subsection{LDA classifier for CIFAR-10}

\begin{figure}[hbt]
\vspace{-0.5cm}
 \centering
\begin{subfigure}[b]{0.45\columnwidth}
\includegraphics[width=\linewidth]{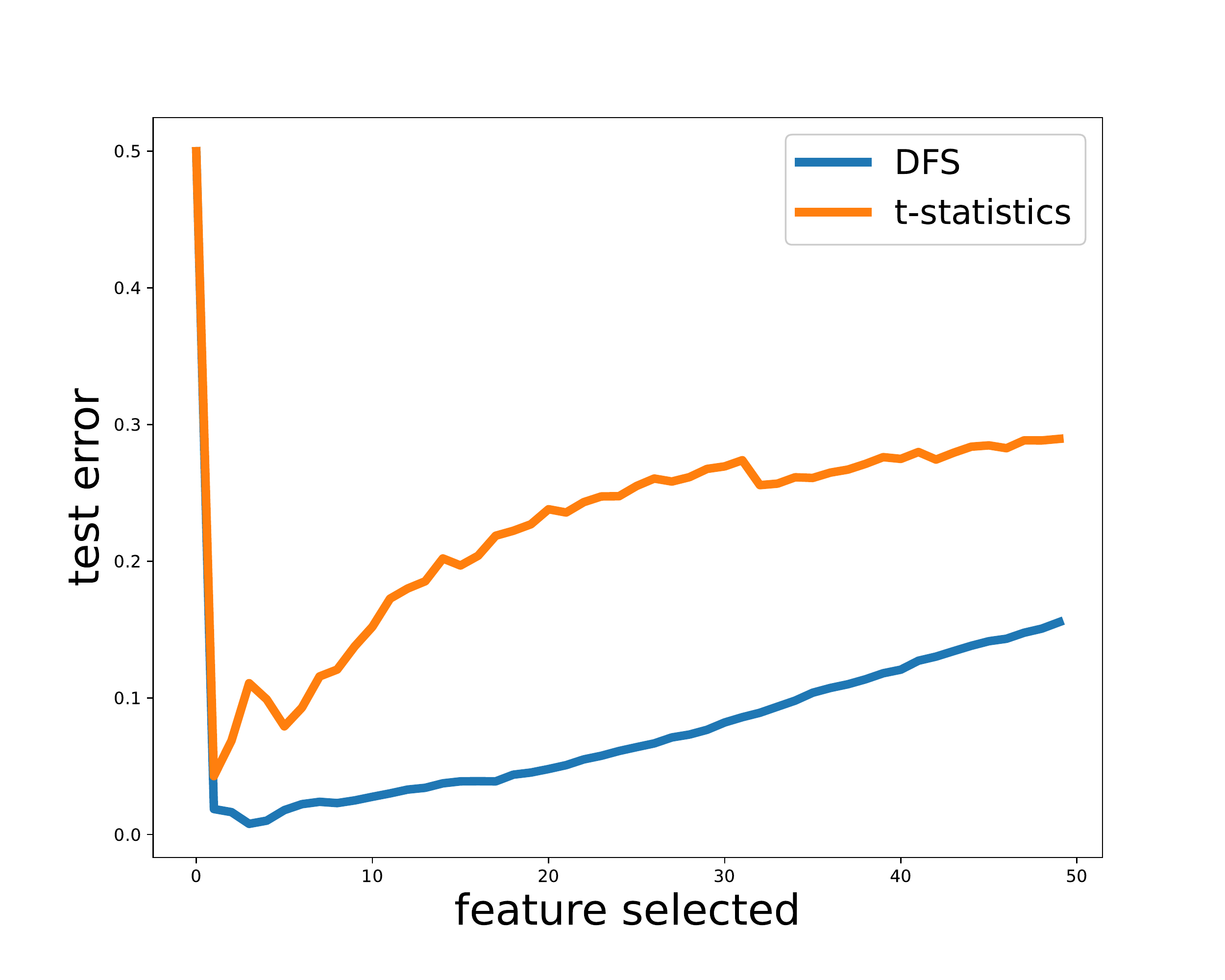}
\caption{fixed $\varepsilon = 6$ with $p$ increasing}\label{fig_CIF}
\end{subfigure}
\begin{subfigure}[b]{0.45\columnwidth}
\includegraphics[width=\linewidth]{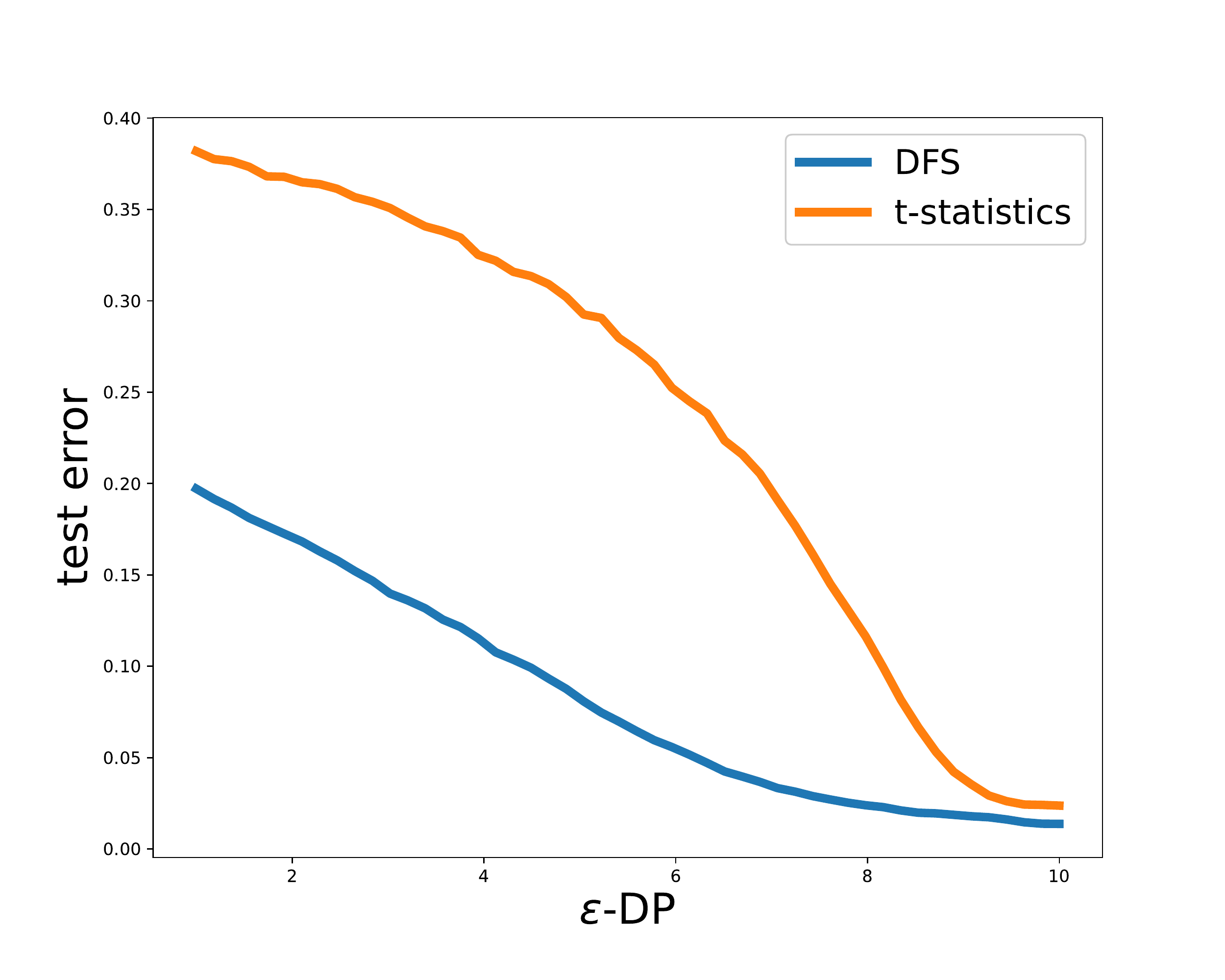}
\caption{fixed $p = 25$ with $\varepsilon$ increasing }\label{fig_cif_cov}
\end{subfigure}
\caption{Results for CIFAR-10}
\end{figure}

For LDA classifier on CIFAR-10, left Fig. \ref{fig_CIF} shows our curve is lower and smoother which means robustness with dimension increasing. Right Fig.\ref{fig_cif_cov} proves when $\varepsilon$ is tiny, noise is large, DFS can perform over t-testing for more than 0.2 in test error.

\subsection{Experiment for MNIST}
\begin{figure}[hbt]
\vspace{-0.5cm}
 \centering
\begin{subfigure}[b]{0.45\columnwidth}
\includegraphics[width=\textwidth]{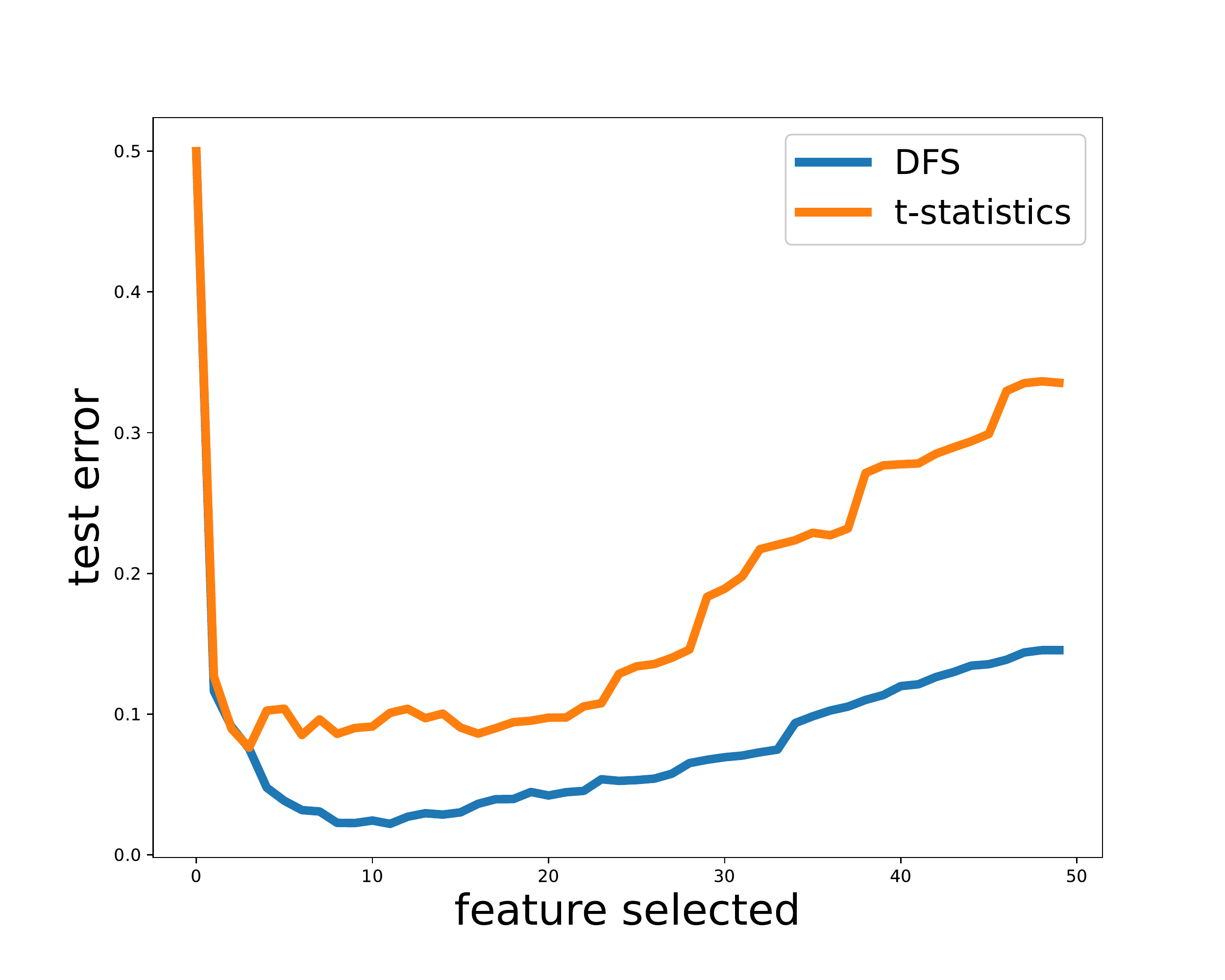}
\caption{fixed $\varepsilon = 6$ with $p$ increasing}\label{fig_MNI}
\end{subfigure}
\begin{subfigure}[b]{0.45\columnwidth}
\includegraphics[width=\textwidth]{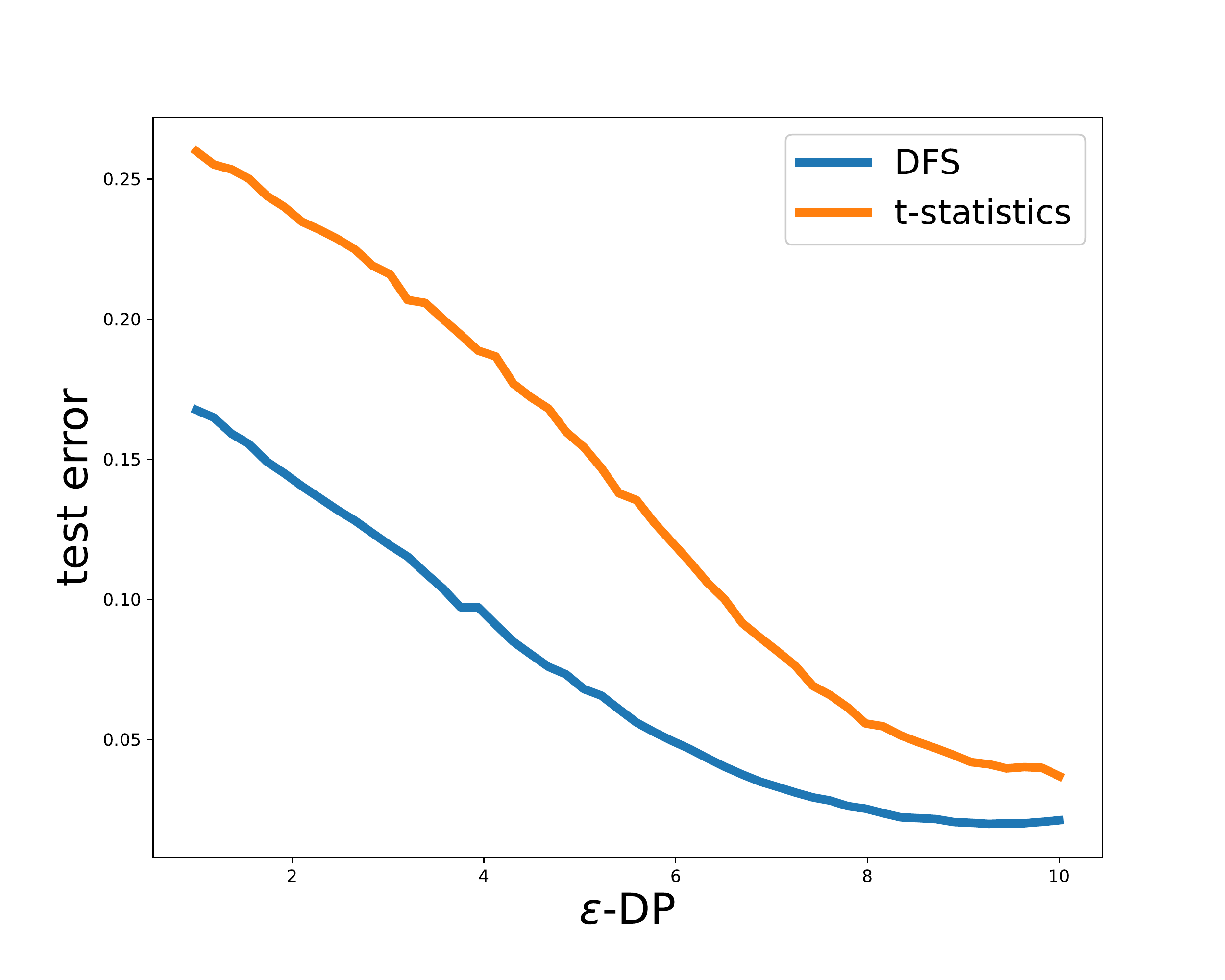}
\caption{fixed $p = 25$ with $\varepsilon$ increasing}\label{fig_mni_cov}
\end{subfigure}

\caption{Results for MNIST}
\end{figure}

Left Fig.\ref{fig_MNI} shows robustness similar to CIFAR-10 since curve is lower and smoother. Right Fig.\ref{fig_mni_cov} shows t-testing is susceptible to DP noise, even $\varepsilon = 10$ would cause error increasing.


\end{document}